%% file: main.tex
\documentclass{article}
\usepackage{iclr2026_conference,times}
\usepackage{graphicx}
\usepackage{amsmath}
\usepackage{amssymb}
\usepackage{amsthm}
\usepackage{booktabs}
\usepackage{algorithm}
\usepackage{algpseudocode}
\usepackage{xcolor}
\usepackage{url}
\usepackage{hyperref}

\graphicspath{{figs/}}

\newtheorem{theorem}{Theorem}
\newtheorem{corollary}{Corollary}

\input{sec/numbers}

\input{figs/co2_macros}

\input{sec/residual_macros}
\input{sec/estimator_macros}

\input{sec/independence_macros}
\input{sec/map_macros}

\input{sec/ci_macros}
\input{sec/iclr_macros}
\input{sec/gpd_macros}

\iclrfinalcopy

\title{Web Scale on a Budget:\\ Tail-Aware Retrieval Evaluation}

\author{Boaz Meivar \\
  Tel Aviv University \\
  \texttt{boazmeivar@mail.tau.ac.il}
  \And
  Shai Avidan \\
  Tel Aviv University \\
  \texttt{avidan@eng.tau.ac.il}}

\begin{document}
\maketitle
\lhead{Preprint}      

\begin{abstract}
A common practice in instance retrieval is to grow the gallery size to improve a benchmark's
discrimination capability. One prominent example is ILIAS, which adds up to 100 million distractor
images. The prevailing assumption is that faithful evaluation demands web scale, and that indexing all
100M images through the backbone is the unavoidable price of a comparable number. That price is steep:
hundreds of GPU-hours and $\sim$200\,GB of descriptors \emph{per model}, paid again for every method or
ablation. We show the web-scale indexing price is largely avoidable: \emph{evaluation itself obeys a
scaling law}. The number of distractors that outrank a positive is Binomial in the gallery size. We
define the per-query retrieval hardness, a gallery-scale agnostic quantity that yields a retrieval
metric at any scale, and demonstrate it can be derived from a small random sample. The underlying sampled-rank machinery is long established in recommender-system and biometric
evaluation, but to our knowledge has never been adapted to visual retrieval. We make the law per-query and estimable, and
validate it at \emph{true} 100M scale on ILIAS and INSTRE, and on the full $759$k-distractor Google
Landmarks v2 pool. Concretely, within a $5\%$ sample (a flat $20\times$ less indexing), every backbone's tail-aware (EVT) estimate has \emph{converged} onto the
large-scale mAP, to within ${\sim}0.01$ (mean absolute error $0.005$), across five backbones, three benchmarks, and two distractor pools (mAP@1000 at 100M on
ILIAS/INSTRE; the official mAP@100 at the full $759$k pool on GLDv2); the recommender-systems count
baseline over-predicts everywhere ($+0.037$ to $+0.068$). This is
an order of magnitude better than naively reporting the metric on the same subset ($+0.08$ to $+0.13$).
We release code and the per-query hardness to facilitate faster, cheaper, and greener retrieval
evaluation.
\end{abstract}

\input{sec/intro}
\input{sec/related}
\input{sec/theory}
\input{sec/protocol}
\input{sec/experiments}

\input{sec/discussion}

\section*{Reproducibility Statement}
All code will be made publicly available.

\section*{Use of Large Language Models}
Large language models were used during the preparation of this paper: to aid and
polish the writing, to draft sections of the paper, to aid retrieval and discovery
of related work, and to aid in proving the paper's mathematical claims. We have
reviewed all AI-assisted work and take responsibility for the final content of this
work, including text, claims, and artifacts produced with the aid of generative AI.

\bibliographystyle{iclr2026_conference}
\bibliography{refs}

\appendix
\input{sec/appendix}
\end{document}

%% file: sec/numbers.tex
\newcommand{\nBackbones}{5}

%% file: figs/co2_macros.tex
\newcommand{\coIndexHours}{556}

\newcommand{\coIndexKg}{89}
\newcommand{\coSweepKg}{622}

\newcommand{\coProtKg}{4.4}

\newcommand{\coReduction}{20}

%% file: sec/estimator_macros.tex
\newcommand{\maeMid}{0.005}

\newcommand{\maeLoguni}{0.038}
\newcommand{\maePower}{0.120}

%% file: sec/map_macros.tex
\newcommand{\trueMap}{0.266}
\newcommand{\mapEvtVal}{0.260}

\newcommand{\mapEvtErr}{0.006}
\newcommand{\mapCountErr}{+0.050}
\newcommand{\mapEvtMaxErr}{0.008}
\newcommand{\maeMap}{0.005}

\newcommand{\instreHundredMap}{0.547}

%% file: sec/iclr_macros.tex
\newcommand{\repK}{20}
\newcommand{\repMae}{0.005}
\newcommand{\repMaeLo}{0.002}
\newcommand{\repMaeHi}{0.008}
\newcommand{\repStd}{0.003}
\newcommand{\repWorst}{0.013}
\newcommand{\repCover}{100/100}

\newcommand{\repPairsPerfect}{9/10}
\newcommand{\repResolveGap}{0.013}
\newcommand{\repContestGap}{0.010}

%% file: sec/gpd_macros.tex
\newcommand{\gpdFail}{0.0\%}
\newcommand{\gpdCapLo}{10\%}
\newcommand{\gpdCapHi}{30\%}
\newcommand{\gpdEndpoint}{0.84}

%% file: sec/intro.tex
\section{Introduction}
\label{sec:intro}

\begin{figure}[t]
\centering
\includegraphics[width=0.62\linewidth]{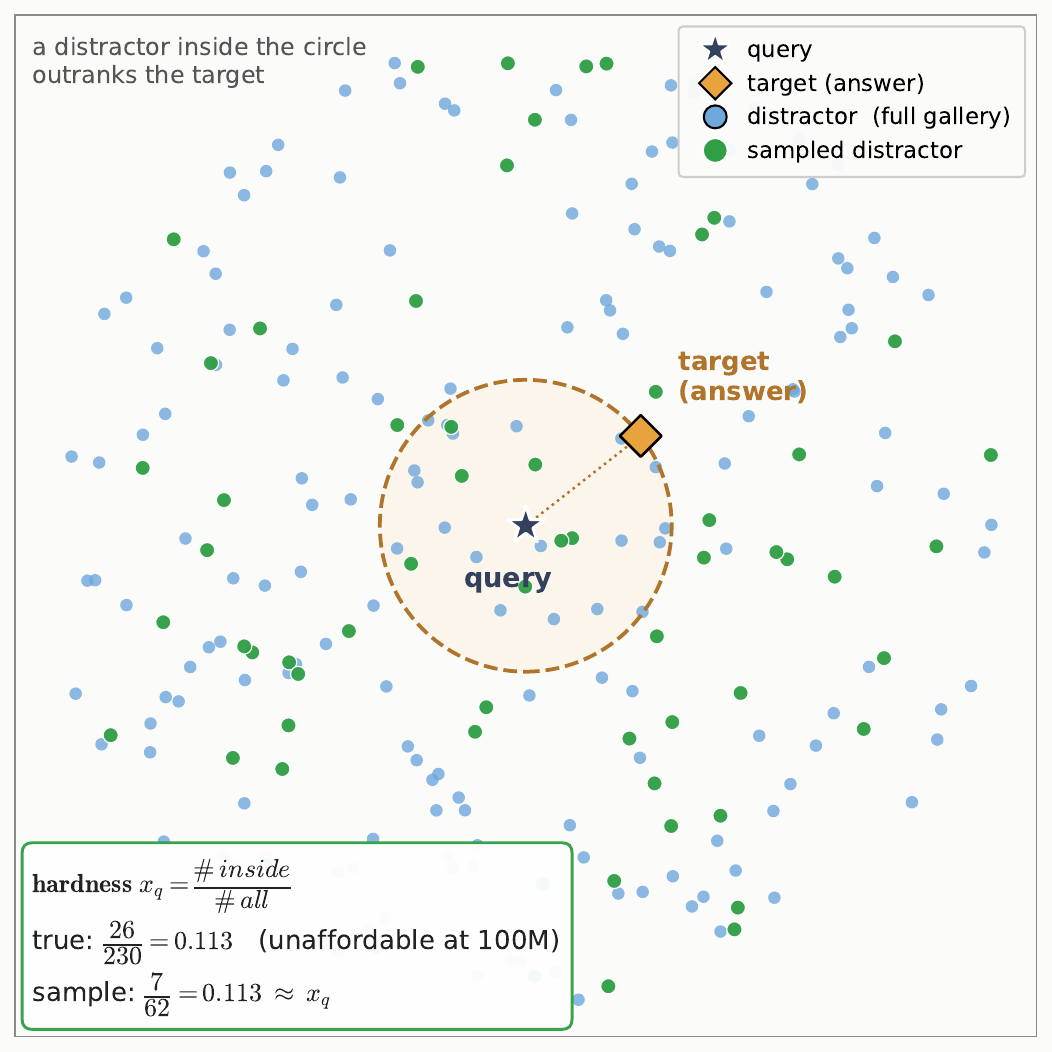}
\caption{\textbf{Why web-scale evaluation is predictable---and how a small sample makes it cheap.}
Retrieval in embedding space for one query ($\star$); its target/answer (the $\blacklozenge$) sits at a
fixed distance (dashed circle). A distractor \emph{inside} the circle scores above the target and outranks
it, so over $N$ distractors the number beating the target is $\mathrm{Binomial}(N,x_q)$ and the target's
rank is $1+\mathrm{Binomial}(N,x_q)$. The per-query hardness $x_q=\#\text{inside}/\#\text{all}$---the
fraction of the distractor distribution inside the circle---is \emph{size-independent}. Measuring it on
the full gallery (every blue point) is the unaffordable 100M indexing; instead we draw a small random
\emph{sample} (green, inside and outside) and estimate $x_q$ from the green points alone
($\hat x_q\approx x_q$). That one number fixes mAP at \emph{any} gallery size $N$, so a $5\%$ sample
predicts the 100M-scale result without indexing 100M distractors. (Illustrative 2-D points; the sample is
drawn larger than $5\%$ for legibility.)}
\label{fig:teaser}
\end{figure}

Progress in instance-level image retrieval is measured by growing the gallery. To separate strong
embeddings from weak ones, benchmarks add large sets of \emph{distractor} images: revisited
Oxford/Paris add a million (R1M)~\citep{radenovic2018revisiting}, Google Landmarks indexes
millions~\citep{weyand2020gldv2}, and ILIAS~\citep{ilias2025} adds one hundred million. This works:
performance keeps dropping as the gallery grows, which is exactly what makes these benchmarks
discriminative. The price, however, is paid in indexing. Encoding 100M images through a
ViT-L backbone costs on the order of \coIndexHours\ GPU-hours and produces $\sim$200\,GB of
descriptors \emph{for one model}. The field treats this as the unavoidable cost of a
meaningful comparison. That cost is paid again for every method, ablation, resolution, or backbone.

Most of that cost is unnecessary. A natural shortcut is to evaluate against a small random subset of
the distractors. This shortcut is distrusted because the resulting number is optimistic and, taken at face
value, can mis-rank methods. So practitioners either pay the full price or retreat to smaller galleries.
These are either ad-hoc random subsets, which are optimistic and not comparable across papers, or
hand-built proxies such as mini-ILIAS~\citep{ilias2025}, which report their own number rather than the
web-scale one.

The shortcut is salvageable because the gap between the small- and large-scale metric is not
mysterious. For a fixed query, the number of distractors that score above the positive is a sum
of independent indicators, hence \emph{Binomial} in the gallery size $N$. The probability the
positive holds rank one is $(1-x)^N$, where the single per-query \emph{hardness} $x$ is the
fraction of the distractor similarity distribution above the positive. This $x$ is
size-independent. It is a property of the model and the distractor source, and it does not depend
on how many distractors one draws. From it, the metric at \emph{any} scale follows in closed form.
The corollary that matters in practice: a small random sample suffices to estimate $x$, and the
large-scale number can be predicted rather than measured.

Why has retrieval not been doing this? The sampled-rank machinery is mature, but it has lived
elsewhere: classical IR~\citep{madigan2003retrieval,hawking2003collection}, recommender
systems~\citep{krichene2020sampled,jin2021estimating}, and
biometrics~\citep{nistfrvt,scheirer2011meta,grother2004models,slavutsky2021predicting,rocher2025scaling}.
That machinery has been used with a single \emph{assumed} score distribution and no sample-based
estimation (e.g.\ \citet{madigan2003retrieval} analyze one shared relevant/irrelevant score pair); what
has been missing for visual retrieval is estimating, \emph{per query and from a small sample}, these
\emph{extremely rare} exceedance probabilities and validating the resulting
metric against a real system at web scale. \textbf{We turn that machinery into a practical, validated
evaluation framework for web-scale instance retrieval}: per-query hardness estimated from a small sample
with a tail-aware estimator, validated at genuine 100M scale, and turned into a low-cost protocol.

\paragraph{Contributions.}
\begin{enumerate}\itemsep1pt \parskip0pt \topsep0pt
\item We bring the distractor-scaling law to instance retrieval and derive its retrieval
form. The positive's rank is $1+\mathrm{Binomial}(N,x_q)$ with per-query hardness
$x_q=1-F_q(s^\star)$, so the reported metric $\mathrm{mAP@}k(N)$ follows in closed form. This is an
\emph{exact identity} in the realized per-query exceedance fractions $\{K_{q,j}/M\}$, since the gallery's
ranks are deterministic. A small sample turns it into a \emph{predictor} through the population hardness,
with error we quantify for the smooth term and evaluate empirically at the top-$k$ boundary
(App.~\ref{app:bound}). The mechanism is broad, multi-decade (approximately power-law) hardness
$\Rightarrow$ near-logarithmic, non-saturating decay (log-uniform is the idealized limit).
\item \textbf{A low-cost web-scale evaluation.} From an affordable $1/20$ distractor subset a tail-aware
(EVT) estimate predicts the web-scale metric. The estimate recovers the large-scale metric to mean
absolute error $0.005$ on ILIAS and ${\le}0.008$ on INSTRE and GLDv2, beating curve-fit and
recommender-systems baselines.
\item \textbf{Validation at true 100M scale.} On ILIAS and INSTRE the estimator reproduces the official
mAP over all 100M YFCC distractors; on Google Landmarks v2 (a different domain with a separate $759$k
pool) it recovers the official mAP@100. It holds across all five tested backbones and every benchmark.
\item \textbf{A public release.} We release code and per-query hardness so results on these benchmarks
can be compared at any gallery size without re-indexing.
\end{enumerate}

%% file: sec/related.tex
\section{Related Work}
\label{sec:related}

We organize prior work by which \emph{axis} of cost it addresses. Training-side scaling laws
predict error from training data/compute. Evaluation-side work either subsamples the
\emph{query} axis or estimates metrics from sampled \emph{judgments}. We occupy a different cell:
predicting the metric along the \emph{distractor/gallery} axis.

\paragraph{Distractor-scaling benchmarks (the cost we remove).}
Revisited Oxford/Paris add the R1M hard-distractor set~\citep{radenovic2018revisiting}, Google
Landmarks indexes millions~\citep{weyand2020gldv2}, and ILIAS scales to 100M
distractors~\citep{ilias2025}. These works \emph{measure} how accuracy falls as the gallery grows
and report a few fixed operating points. None make the metric-versus-scale curve estimable from a
sample or use it to cut evaluation cost. We do. The ILIAS authors do release \emph{mini-ILIAS}, a
VLM-mined 5M-distractor proxy that retains the hardest distractors as a cheaper stand-in for the full
pool~\citep{ilias2025}. It reports a number on a deliberately harder subset, whereas we \emph{predict}
the true 100M number from a uniform random sample. A hard-mined subset is non-uniform by design, so it
does not yield the size-independent hardness our estimator needs.

\paragraph{Collection-size analysis in classical IR (closest in spirit).}
The relationship between collection size and retrieval effectiveness was modeled analytically two
decades ago. \citet{madigan2003retrieval} derive the distribution of irrelevant documents
outscoring the relevant ones (``contamination at $K$''). This is a Binomial/Poisson sampled-rank
model whose $K{=}1$ case is our Theorem~\ref{thm:main}. Theirs assumes a \emph{single} score
distribution shared by all queries, with no sample-based estimation (which they pose as future work).
The same authors later applied extreme-value theory to the growing-collection asymptotics of retrieval
effectiveness~\citep{madigan2006evt}, still for a single shared score distribution and without
sample-based estimation or 100M-scale validation.
\citet{hawking2003collection} empirically observe that mAP rankings are largely preserved between a
sample and the full collection. We supply what this line lacked: the heterogeneous per-query law, a
sample-based tail-aware estimator, and validation at 100M.

\paragraph{Sampled / low-cost retrieval evaluation (judgment axis).}
A long IR tradition estimates metrics from incomplete or sampled relevance judgments:
bpref~\citep{buckley2004incomplete}, statAP/infAP~\citep{yilmaz2008statap,yilmaz2006infap}, and
minimal test collections~\citep{carterette2006minimal}. In recommendation, Krichene and
Rendle~\citep{krichene2020sampled} show sampled top-$k$ metrics are biased and collapse toward
AUC, and propose corrected estimators (an unbiased rank estimator and a bias--variance variant); a
subsequent line learns the global rank \emph{distribution} by maximum
likelihood~\citep{jin2021estimating,li2023reliable}. These estimate an \emph{aggregate} rank
distribution (Recall/NDCG@$k$ over a catalogue) and cure the sampling ``blind spot'' by sampling
\emph{more} items, not by extrapolating the unobserved tail. Our count baseline
(Sec.~\ref{sec:experiments}) is the Krichene--Rendle \emph{unbiased} estimator specialized to the top
rank; because mAP depends on the per-positive hard tail (invisible below $1/S$) it fails to recover the
large-$N$ value, and enlarging the sample to resolve that tail is exactly the web-scale indexing we
avoid---which is why we instead extrapolate it.

\paragraph{Efficient benchmarking (query axis).}
tinyBenchmarks~\citep{polo2024tiny} and Anchor Points~\citep{vivek2024anchor} rank models from a
handful of test items via item-response theory or clustering. These reduce the number of
\emph{queries}. We reduce the number of \emph{distractors}, an orthogonal and, at web scale,
dominant cost.

\paragraph{Identification scaling in biometrics.}
Identification rates scale with gallery size~\citep{nistfrvt}, and the MegaFace benchmark measures
face-identification accuracy as the gallery grows to a million distractors~\citep{megaface2016}.
Meta-recognition fits an extreme-value tail to match scores~\citep{scheirer2011meta}. The same generalized-Pareto
peaks-over-threshold object appears in IR for list truncation~\citep{bahri2020surprise}, where the
tail exceedance probability is exactly our $x_q$. The distinction is what the fit is \emph{for}:
prior work uses it at a fixed scale, whereas we use it as an \emph{extrapolation engine} to a far
larger gallery. Most related, \citet{rocher2025scaling} forecast large-population identification
correctness from small experiments with a Pitman--Yor law. We differ in task (instance retrieval
vs.\ closed-set identification), model, and scale (100M images). We read their result as concurrent
evidence that evaluation-side scaling laws are learnable.

\paragraph{Training-side scaling laws (contrast).}
Power-law generalization in training data and model size~\citep{hestness2017scaling,kaplan2020scaling},
including for dense retrieval~\citep{fang2024scaling}, brought predictive discipline to \emph{training}
budgets. We bring the same discipline to the \emph{evaluation} budget, treating evaluation scaling as a
first-class problem in its own right, analogous to training scaling laws. A parallel line extrapolates
multi-class \emph{classification} accuracy to larger label sets from a subset of
classes~\citep{zheng2018extrapolating}, the class-axis analogue of our gallery-axis extrapolation.

%% file: sec/theory.tex
\section{A Scaling Law for Retrieval Under Distractors}
\label{sec:theory}

\paragraph{Setup.}
Fix a query $q$ and rank all gallery items by descending cosine similarity to $q$. Let $s^\star$
be the similarity of the query's best positive (the maximum over its relevant images). Let
$d_1,\dots,d_N$ be $N$ distractors drawn i.i.d.\ from the distractor source, with similarities
$S_i=\mathrm{sim}(q,d_i)$ and common CDF $F_q$. Define the \emph{per-query hardness}
\begin{equation}
\label{eq:hardness}
x_q \;:=\; \Pr\!\big[\mathrm{sim}(q,d)>s^\star\big] \;=\; 1-F_q(s^\star),
\end{equation}
the fraction of the distractor similarity distribution lying above the best positive. The hardness
$x_q$ is a \emph{population} quantity: a property of the model and the distractor \emph{source}, not
of the number $N$ of distractors drawn. It is distinct from the \emph{realized} exceedance fraction in
a particular gallery. We make that distinction precise below, and it governs when the law is exact.

\begin{theorem}[Distractor scaling]\footnote{This is the instance-retrieval form of a
sampled-rank model that has been derived independently in classical
IR~\citep{madigan2003retrieval}, recommender systems~\citep{krichene2020sampled}, and
biometrics~\citep{nistfrvt}. Madigan et al.'s ``contamination at $K$'' is the single-positive case. Our
contribution is not the probabilistic fact but the \emph{heterogeneous} (per-query) law that fits
data where the single-distribution model fails, together with the sample-based estimator, the 100M
validation, and the protocol.}
\label{thm:main}
Let $K=\#\{i:S_i>s^\star\}$ be the number of distractors outranking the positive. Then
$K\sim\mathrm{Binomial}(N,x_q)$, the positive sits at rank $1+K$, and
\begin{equation}
\begin{gathered}
\Pr[\text{positive ranked first}] = (1-x_q)^N,\\
\mathbb{E}[\mathrm{rank}] = 1+N\,x_q .
\end{gathered}
\end{equation}
For a query with positives $1,\dots,n_q$ at descending similarity, let
$x_{q,j}=\Pr[\mathrm{sim}(q,d)>s_j]$ be the per-positive (population) hardness. Collecting these,
define the \emph{hardness-to-metric map}
\begin{equation}
\label{eq:aplaw}
\begin{aligned}
R_k\big(\{u_{q,j}\},N\big) &=\frac{1}{|Q|}\sum_q\frac{1}{\min(k,n_q)}\\
&\quad\times\sum_{j=1}^{n_q}\mathbf{1}\{r_j\le k\}\,\frac{j}{r_j},\\
r_j&=j+N\,u_{q,j}.
\end{aligned}
\end{equation}
As shown next (and in App.~\ref{app:bound}), $R_k$ is an \emph{exact identity} for the realized
gallery when fed its realized exceedance fractions, and a \emph{predictor} when fed the population
hardness $\{x_{q,j}\}$---one whose smooth-precision error we quantify and whose top-$k$ truncation error we
evaluate empirically (App.~\ref{app:bound}).
Each summand $j/r_j$ is non-increasing in $N$, so $\mathrm{mAP@}k(N)$ is non-increasing: a larger
gallery never raises the metric (by coupling, every $K$ only grows as distractors are added). In the
rare-event regime ($N$ large, $x_q$ small, $Nx_q=O(1)$), $K\to\mathrm{Poisson}(Nx_q)$ and the
top-rank survival $\Pr[\text{positive ranked first}]\approx e^{-Nx_q}$. A positive holds the top-$k$
until $N\approx k/x_{q,j}$.
\end{theorem}

\begin{proof}
The indicators $B_i=\mathbf{1}\{S_i>s^\star\}$ are i.i.d.\ $\mathrm{Bernoulli}(x_q)$, so their sum
$K$ is $\mathrm{Binomial}(N,x_q)$. The positive is at rank $1+K$, so it is ranked first iff $K{=}0$,
with probability $(1-x_q)^N$. By linearity of expectation (which requires \emph{no} independence),
$\mathbb{E}[\mathrm{rank}]=1+\mathbb{E}[K]=1+\sum_i\mathbb{E}[B_i]=1+Nx_q$, and more
generally the $j$-th positive (already preceded by $j{-}1$ better positives) sits at expected rank
$j+N\,x_{q,j}$. The map~\eqref{eq:aplaw} assembles these per-positive ranks into the metric. Whether
it equals the realized mAP exactly or only predicts it depends on which hardness it is fed, which we
make precise after the proof. The Poisson limit is the standard convergence
$\mathrm{Binomial}(N,x_q)\to\mathrm{Poisson}(Nx_q)$. Monotonicity follows by coupling
$K_{N+1}=K_N+B_{N+1}\ge K_N$, so each rank $r_j$ only grows and each $j/r_j$ only shrinks.
\end{proof}

\paragraph{Exact identity vs.\ prediction.}
The map~\eqref{eq:aplaw} is used two ways, and separating them is what makes ``exact law'' precise.
\emph{(i) Identity:} a benchmark pool is one fixed set, and the $j$-th positive's rank is the
deterministic $r_j=j+K_{q,j}$, where $K_{q,j}$ counts the distractors outscoring it. With the
\emph{realized fraction} $\bar x_{q,j}:=K_{q,j}/N$ we have $j+N\bar x_{q,j}=r_j$, so
$R_k(\{\bar x_{q,j}\},N)$ is that gallery's mAP \emph{exactly}: truncation included, no expectation, no
linearization (App.~\ref{app:bound}). \emph{(ii) Prediction:} the population hardness $x_{q,j}$ of
Eq.~\eqref{eq:hardness} differs. Since $K_{q,j}\sim\mathrm{Binomial}(N,x_{q,j})$, $\bar x_{q,j}$
concentrates on it with standard deviation $\sqrt{x_{q,j}(1{-}x_{q,j})/N}$, so feeding $\{x_{q,j}\}$
\emph{predicts} the realized mAP. That prediction error is empirically $\le0.006$ at 100M. The
$\mathbb{E}[1/r]$-vs-$1/\mathbb{E}[r]$ gap concerns only the never-reported pool-resampling average
$\mathbb{E}[\mathrm{AP}]$, which we compute exactly (the per-positive Binomial sum
$\sum_m\tfrac{j}{j+m}\Pr[K_{q,j}{=}m]$) and find sits ${\approx}0.005$ \emph{above} the realized value:
since a benchmark reports the \emph{realized} number, targeting it with the expected-rank map---as we
do---is more accurate than targeting $\mathbb{E}[\mathrm{AP}]$, not less. Our estimator plugs in a
size-$S$ sample estimate
$\hat x_{q,j}$, splitting the error into estimation (sample$\to$population) and concentration
(population$\to$realized), both quantified in App.~\ref{app:bound}.

\paragraph{Remarks.}
The expected-rank result $1+Nx_q$ holds under \emph{arbitrary} dependence between distractors. The
exact $(1-x_q)^N$/Poisson forms assume independence, which holds for i.i.d.\ random distractors and
which we find empirically tight at 100M (Sec.~\ref{sec:experiments}). The theorem models $N$ i.i.d.\
draws from a distractor \emph{source}. Evaluating on a fixed release (e.g.\ the official 100M set) is
one realization, for which the law is the with-replacement approximation of the finite-pool
(hypergeometric) outcome. The two coincide at sampling fractions $\ll1$, and the full-pool value is the
$N\!\to\!\text{pool-size}$ limit. Everything hinges on the single per-query scalar: the law is an exact
identity at the realized exceedance fractions and a concentration-tight predictor at the population
hardness (above), and all of the empirical content lives in how that hardness is distributed across
queries.

\paragraph{Logarithmic decay under log-uniform hardness.}
\begin{corollary}[Log-uniform hardness $\Rightarrow$ logarithmic decay]
\label{cor:loguniform}
If the hardness density is log-uniform, $g(x)\propto 1/x$ on $[a,b]$, then the aggregate top-rank
survival
\begin{equation}
\begin{aligned}
\mathbb{E}_q\big[(1-x_q)^N\big]
&=\int(1-x)^N g(x)\,dx\\
&\approx\frac{E_1(Na)-E_1(Nb)}{\ln(b/a)} \approx A-B\ln N
\end{aligned}
\end{equation}
for $1/b\ll N\ll 1/a$, where $E_1$ is the exponential integral. Each $j/r_j$ term
in~\eqref{eq:aplaw} decays the same way, so $\mathrm{mAP@}k(N)$ inherits the logarithmic profile.
Accuracy therefore decays \emph{logarithmically and never plateaus} over the relevant range.
\end{corollary}
This is an empirical regularity of the datasets we study, not a universal law. Across all of them the
per-query hardness is \emph{approximately} a power law $x^{-\alpha}$ with $\alpha\approx1.2$
($1.16$--$1.32$ across our backbones and benchmarks; Sec.~\ref{sec:experiments}, Fig.~\ref{fig:law}); the
log-uniform case above ($\alpha{=}1$) is its idealized limit, and the mild power law induces the observed
near-logarithmic decay: mAP falls by a roughly constant amount per decade of distractors all the way to
100M, with a large fraction of mass near the
detection floor $x\!\approx\!1/N$. It is the precise sense in which evaluating on a small subset is
optimistic: the subset has not yet revealed the slow bleed of the hardness tail.

%% file: sec/protocol.tex
\section{From Law to Protocol}
\label{sec:protocol}

Theorem~\ref{thm:main} says a model's behavior at any gallery size is determined by the
size-independent hardness distribution $\{x_q\}$. Evaluation therefore need not scale with the
distractor set: estimate $\{x_q\}$ from a small random sample and report the metric (mAP@$k$) at a
standardized target $N^\star$ via Eq.~\eqref{eq:aplaw}, so numbers obtained with different
distractor budgets become directly comparable.

\paragraph{Estimating hardness from a size-$S$ sample.}
Draw $S$ random distractors and, per query, record $s^\star$ and the sample similarities. The
difficulty is that the queries which dominate large-$N$ accuracy have hardness $x_q\!\lesssim\!1/S$
and are \emph{invisible} in the sample. We consider:
\emph{(i) Count} $\hat{x}_q=\tfrac1S\#\{S_i>s^\star\}$ -- unbiased but quantized at $1/S$; it
returns $0$ for the invisible tail and so predicts those queries survive forever. This is exactly
the Krichene--Rendle \emph{unbiased} rank estimator specialized to the top rank, and it is over-optimistic.
\emph{(ii) Tail extrapolation via extreme-value theory (EVT)}~\citep{scheirer2011meta,bahri2020surprise}
-- for these queries no sampled distractor reaches $s^\star$, so counting is blind to them. Extreme-value
theory states that above a high threshold the far tail of essentially any continuous score distribution
converges to a universal two-parameter shape, a generalized Pareto. We fit that shape to the
highest-scoring distractors we \emph{do} observe and extend it out to $s^\star$, recovering a small
nonzero hardness. The move is the same one used to estimate a never-observed ``100-year flood'' from the
tail of recorded flood heights (full procedure in App.~\ref{app:evt}).
\emph{(iii) Survival-curve fit} -- fit the mechanism of Cor.~\ref{cor:loguniform}, the log-uniform
survival model $S_1(N)=w_{\mathrm{imm}}+w_c\,[E_1(Na)-E_1(Nb)]/\ln(b/a)$ for the top-rank survival,
to the measured survival curve $S_1(N{\le}S)$ and extrapolate.

\paragraph{The estimate and its interval.}
The two families fail on the \emph{same} object, the hardness mass below the sample floor $1/S$, in
opposite one-sided ways. Counting and curve fits assign that mass zero, so they can only over-predict,
with positive error in every one of our five backbone cases. Tail extrapolation instead
assigns the mass a fitted value and recovers it. Because mAP integrates marginal positives whose
hardness lives precisely in that unobserved tail, we take the tail-aware (EVT) estimate
$\hat{R}_{\mathrm{tail}}$ as the point estimate. It is accurate across three benchmarks, two
distractor pools, and five backbones. The count estimate $\hat{R}_{\mathrm{count}}$ zeroes the unresolved
sub-$1/S$ tail mass, which can only \emph{raise} the predicted metric---a one-sided bias on that mass,
not a deterministic per-sample guarantee; the tail estimate has no matching bias in the other direction,
so we treat the gap $B=\hat{R}_{\mathrm{count}}-\hat{R}_{\mathrm{tail}}$ as an
\emph{empirical diagnostic}---computable from the sample and shrinking as $S$ grows---not a coverage
interval. Empirically the truth sits at the tail edge in every case we measure, but a misspecified tail
could in principle err either way.
Alongside it we report a query-bootstrap confidence interval that captures sampling noise. Judged
against this interval, the estimate enters the metric's own noise from a small sample
(Sec.~\ref{sec:experiments}). An error analysis (App.~\ref{app:bound}) shows the AP law is exact in
the hardnesses for the realized gallery, so the error is \emph{entirely} the sampling term. That term
concentrates for sample-resolved positives and is governed by extreme-value consistency for the
unresolved tail.

The recipe, summarized in Alg.~\ref{alg:protocol}: sample $S$, fit the per-query EVT tail, report
$\frac1{|Q|}\sum_q\mathrm{AP}_q(N^\star;\hat{x}_{q,\cdot})$ at the target scale via
Eq.~\eqref{eq:aplaw} with the empirical count/tail diagnostic and a query-bootstrap confidence interval.

\begin{algorithm}[t]
\footnotesize
\caption{\textbf{Predicting the web-scale metric from a small distractor sample.} A correct answer's
rank grows by the number of distractors that outscore it---about $N$ times its per-query \emph{hardness}
$x$, the fraction of \emph{all} distractors scoring above the answer. Estimating $x$ from a small sample
therefore fixes mAP at any gallery size $N$, with no full-scale indexing.}
\label{alg:protocol}
\begin{algorithmic}[1]
\Require query set $Q$ with correct answers; a uniform random sample of $S$ distractors;
target gallery size $N$; rank cutoff $k$
\Ensure predicted mAP@$k$ at gallery size $N$
\State construct a uniform (hash-random) sample; optionally check shard consistency for drift
\For{each query $q$}
  \For{each correct answer $p$ of $q$, with similarity $s_p$}
    \State \emph{estimate the hardness} $x$ from the fraction of sampled distractors scoring above $s_p$
           \Comment{exact rule---count edge with a bounded-EVT tail---in \S\ref{sec:protocol} / App.~\ref{app:evt}}
    \State $r_p \gets (p\text{'s rank among the answers}) + N\,x$
  \EndFor
  \State $\mathrm{AP}_q \gets$ average precision from the ranks $\{r_p\}$ \ [Eq.~\eqref{eq:aplaw}]
\EndFor
\State \Return $\mathrm{mAP@}k = \tfrac{1}{|Q|}\sum_{q} \mathrm{AP}_q$ with a query-bootstrap confidence interval
\end{algorithmic}
\end{algorithm}

\paragraph{Requirement: the sample must be i.i.d.}
The law assumes the $S$ distractors are drawn uniformly at random from the pool. This is the one
assumption a practitioner can silently violate. A convenience sample, such as the first $S$ files of a
pool ordered by crawl date, cluster, or upload, is \emph{not} uniform: correlated images concentrate
the score tail and bias the hardness estimate. Hash-based or shuffled sampling is the actual guarantee
that the drawn sample represents the omitted pool. The per-shard check---split the sample into feature
shards and confirm per-shard accuracy is constant (Sec.~\ref{sec:experiments})---is a diagnostic that
catches drift \emph{within} the drawn sample (visible as shard-to-shard variation), not a proof of
representativeness.

\paragraph{When extrapolation is needed.}
The estimator's job is to recover the unobserved hard tail, so its value scales with how much that
tail matters. With a hard pool (large decay), the count estimate is badly optimistic and tail modeling
is essential. With an easy pool (little decay), there is almost nothing to extrapolate and counting
already suffices. Across our datasets a random small sample recovers mAP to within $\sim$0.01.

%% file: sec/experiments.tex
\section{Experiments}
\label{sec:experiments}

\begin{figure*}[t]\centering
\includegraphics[width=\linewidth]{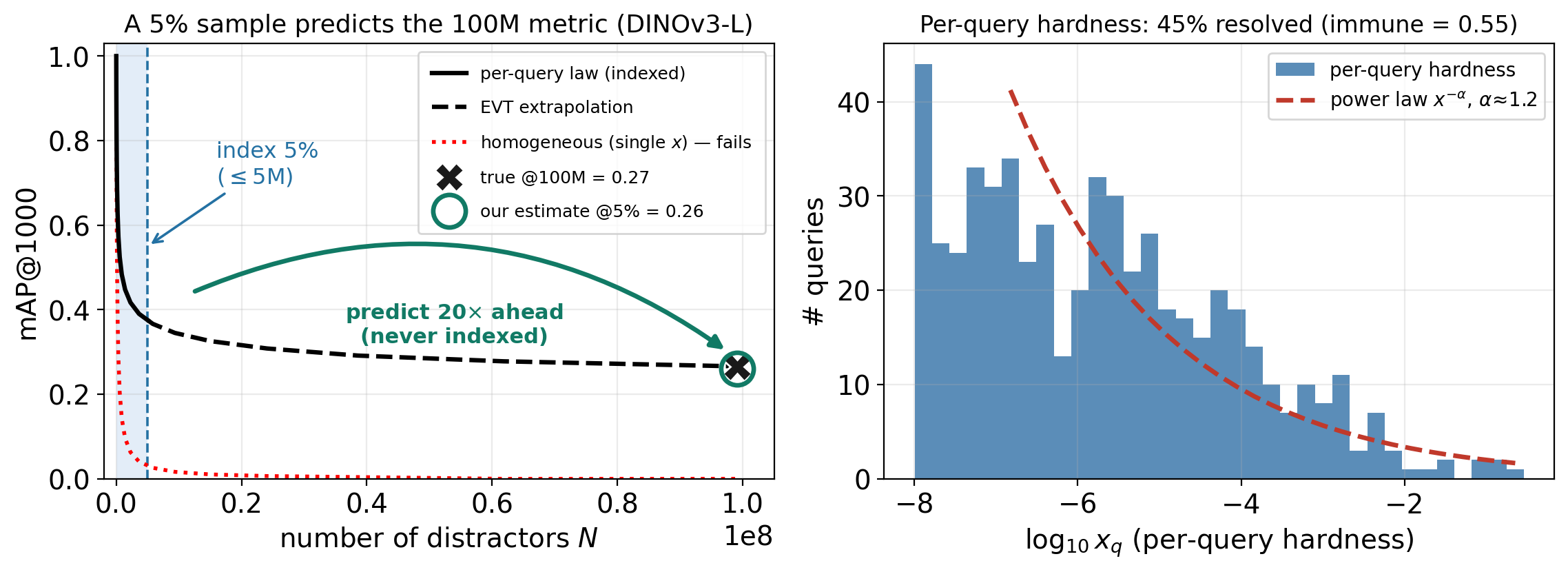}
\caption{\textbf{A $5\%$ sample predicts the $100$M metric.} \emph{Left:} mAP@1000 against gallery size
$N$ for DINOv3-L on ILIAS. Only the shaded $5\%$ is indexed ($N\!\le\!5$M, solid law); past it the
per-query law of Eq.~\eqref{eq:aplaw} is pure EVT extrapolation (dashed), carrying the metric $20\times$
to the true $100$M value (\textbf{X}, $0.27$), which the $5\%$ estimate recovers ($0.26$)---while a
homogeneous single-$x$ model collapses to zero. \emph{Right:} the per-query hardness distribution---$55\%$
of queries stay immune (never outranked even at 100M) and the resolved rest follow an approximate power
law $x^{-\alpha}$, $\alpha{\approx}1.2$ (dashed) across ${\sim}6$ decades toward the $1/N$ floor.}
\label{fig:law}
\end{figure*}

\paragraph{Setup.}
We evaluate on three benchmarks across two distractor pools. \textbf{ILIAS}~\citep{ilias2025}:
1{,}232 instance queries against 100M YFCC distractors, using the official precomputed features
for \nBackbones{} backbones (DINOv2-L, DINOv3-L, SigLIP-L, SigLIP2-L, PE-core-L), which we
reproduce to cosine $1.00$; we report the official mAP@1000. \textbf{INSTRE}~\citep{instre2015}:
1{,}250 object queries with the same 100M pool appended. \textbf{Google Landmarks
v2}~\citep{weyand2020gldv2}: 1{,}129 retrieval queries against the $759$k-image index pool---a
different domain, its own ground truth, and a \emph{same-domain} (landmark-vs-landmark) distractor
distribution---encoded with the same five backbones at their native preprocessing; we report the
official mAP@100. Samples are uniform-random, verified with the per-shard diagnostic of
Sec.~\ref{sec:protocol}.

\paragraph{The law holds (Fig.~\ref{fig:law}).}
For DINOv3-L the per-query prediction tracks measured mAP across eight decades to
\trueMap{} at 100M (the value computed on the full official feature release). By contrast, a
homogeneous single-$x$ model collapses to zero. Recovering the official number end-to-end
validates the pipeline and the law at ground truth. Hardness follows an approximate power law $x^{-\alpha}$
($\alpha\approx1.2$; $1.16$--$1.32$ across the five backbones and three benchmarks): a large fraction of
queries stay near-immune even at 100M, and the resolved rest bleed accuracy at a near-constant rate per
decade---log-uniform ($\alpha{=}1$) is the idealized limit of Cor.~\ref{cor:loguniform} and gives exactly
that. This never-saturating decay is also why reading the metric off a 50k subset is optimistic by $+0.34$
at 100M (Fig.~\ref{fig:law}).

\paragraph{The independence assumption is benign in practice (Fig.~\ref{fig:cross}).}
The exact $(1-x)^N$ form assumes independent distractors. Empirically the residual between the independent
prediction and the measured incremental curve is at most a few percent at 100M across all \nBackbones{}
backbones, and the prediction is stable to $0.004$ across three disjoint $5$M sub-pools
(App.~\ref{app:evt}). Analytically, near-duplicate clusters of size $c$ leave the \emph{expected} rank
unchanged ($1+Nx_q$, clustering-invariant by linearity) and inflate only its variance by ${\sim}c$; on
the real crawl ordering the measured curve sits at $c\!\approx\!1$, so we do \emph{not} claim i.i.d.\ is
conservative for mAP in general---only that the residual is small here (Fig.~\ref{fig:cross},
App.~\ref{app:evt}).

\paragraph{Predicting mAP@100M (Tab.~\ref{tab:map}).}
Because mAP@1000 depends on the rank of \emph{every} positive and marginal positives have hardness in the
unobserved tail, recovering it takes tail extrapolation rather than counting (closed form
Eq.~\eqref{eq:aplaw}, exact in the realized hardnesses; derivation in App.~\ref{app:map},~\ref{app:bound}).
From a $5\%$ sample ($5$M for ILIAS, a $20\times$ reduction), the tail-aware (EVT) estimator recovers true
mAP@100M to mean error \maeMap{} (max \mapEvtMaxErr) across the five ILIAS backbones. The query-bootstrap interval's overlap with the true interval for every backbone (App.~\ref{app:bound}, Tab.~\ref{tab:mapci})
is a weaker check, since prediction and truth share the same queries and their intervals are correlated.
Two stronger checks (App.~\ref{app:bound}) agree: the estimate preserves the backbone ranking exactly
(Kendall $\tau=1.0$, no pairwise sign flip; gaps down to ${\approx}\repResolveGap$ mAP are resolved, and
the closest pair---PE-core vs.\ SigLIP2, gap $0.010$---is within the extrapolation error and should be
read as a statistical tie), and the paired per-query error has pooled $95\%$ CI
$[-0.001,+0.004]$. At this same 5M cost, naively evaluating on the subset is optimistic by mean $0.11$,
whereas our prediction is within $0.005$, roughly $22\times$ more accurate (App.~\ref{app:bound},
Tab.~\ref{tab:eval5m}). For DINOv3-L the estimator gives \mapEvtVal{} vs.\ \trueMap{} (err \mapEvtErr), while the
count estimator over-predicts (\mapCountErr). The estimate is stable across sample size (Fig.~\ref{fig:sens}),
and $5\%$ (5M) is accurate for every backbone across all three benchmarks.

These intervals resample queries only. Folding in the other two noise sources, namely \emph{which}
distractors land in the sample and the per-query EVT refit, needs a nested bootstrap over all five
backbones. It adds only $\sigma\!\approx\!0.003$ of distractor-sample-plus-fit noise, so the intervals
shift slightly rather than widen, and still covers the true 100M value in $5/5$ cases
(App.~\ref{app:bound}, Tab.~\ref{tab:fullci}). Query sampling dominates the uncertainty. Resampling the
distractors directly---$\repK$ disjoint $5\%$ draws per backbone---places \emph{all} $100$ predictions
inside the reported interval (coverage $\repCover$; std $\repStd$, worst error $\repWorst$;
App.~\ref{app:repeated}). The reported interval is a query confidence interval; the remaining noise
sources are quantified in App.~\ref{app:bound}.

We compare against the baselines a practitioner would reach for, all from the same $5\%$ sample. Three
fail on the unresolved tail: reporting the metric on the sample (``evaluate-at-$S$'') is optimistic by
$0.08$--$0.13$ (Tab.~\ref{tab:eval5m}); the tail-blind count (the Krichene--Rendle unbiased estimator),
which assigns the tail below $1/S$ zero mass, over-predicts by $0.037$--$0.068$; and power-law ($\maePower{}$) and log-uniform
survival ($\maeLoguni{}$) curve fits do no better. The tail-aware estimate (mean error $\maeMid{}$; the
headline paired mean of Tab.~\ref{tab:map} is $\maeMap{}$) beats all of them---only the bounded GPD
($\xi\le0$, free scale) excludes the power-law regime and recovers the tail.

\begin{table}[t]\centering\small
\resizebox{\columnwidth}{!}{\input{sec/map_table}}
\caption{\textbf{Recovering the 100M metric from a $5\%$ ($5$M) sample.} For each backbone: the true
mAP@1000 at 100M (our recompute, matching the official ILIAS leaderboard to ${\le}0.001$), and four
numbers obtainable at a $5$M ($20\times$ smaller) budget---the leaderboard's official \emph{hard-mined}
$5$M column, a \emph{naive uniform-random} $5$M evaluation, the tail-blind \emph{count} baseline, and our
tail-aware \emph{EVT} estimate. All three baselines over-predict (mean $+0.041/+0.113/+0.057$); only EVT
recovers the true value (mean $|$err$|$ $0.005$), and it needs no reference 100M numbers---unlike a fixed
$5$M-to-$100$M rescaling of the leaderboard column, which INSTRE and GLDv2 lack. \textbf{All differences
are statistically significant}: each baseline's $95\%$ query-bootstrap interval lies above the truth
($0/5$) and every EVT interval overlaps it ($5/5$); detailed intervals in App.~\ref{app:bound}
(Tab.~\ref{tab:mapci}).}
\label{tab:map}
\end{table}

\paragraph{Generalization across pools (Tab.~\ref{tab:cross},~\ref{tab:baselines}, Fig.~\ref{fig:gldv2}).}
The estimator is not tuned to ILIAS: its configuration ($T{=}4000$, $q{=}0.95$, $\xi\le0$) is fixed once,
so INSTRE and GLDv2 are held-out tests (App.~\ref{app:evt}). On \textbf{INSTRE}
(mAP@1000 at 100M, all \nBackbones{} backbones) it recovers the 100M value to within $0.007$ for every
backbone (mean $|$err$|$ $0.005$, Tab.~\ref{tab:instre}), preserving the ranking, while the count edge
over-predicts (mean $+0.008$). \textbf{GLDv2}
---a different domain with $759$k \emph{same-domain} landmark distractors and official mAP@100---is a
held-out test of comparable difficulty (immune fraction---queries never outranked even at 100M---$0.40$--$0.64$,
vs.\ $0.37$--$0.55$ on ILIAS): the law holds (proof-vs-measured MAE $\le0.010$), the
truth lies inside the interval $5/5$, and a $5\%$ sample recovers mAP@100 to mean error $0.008$
(Fig.~\ref{fig:gldv2}), while count and curve fits over-predict (Tab.~\ref{tab:baselines})---it still
holds. Across all three benchmarks and two pools the recovery stays within ${\sim}0.01$
(Tab.~\ref{tab:cross}).

\paragraph{Faster, cheaper, greener (Fig.~\ref{fig:co2}).}
Indexing 100M images through a ViT-L costs $\approx$\coIndexHours\,GPU-hours and
$\approx$\coIndexKg\,kg\,CO$_2$eq per backbone (A100, $0.4$\,kg/kWh~\citep{lacoste2019co2}), rising to
$\approx$\coSweepKg\,kg for a 7-backbone sweep; the $5\%$ protocol cuts this $\coReduction\times$
(Fig.~\ref{fig:co2}).

%% file: sec/map_table.tex
\setlength{\tabcolsep}{5pt}
\begin{tabular}{lccccc}
\toprule
 & & \multicolumn{4}{c}{estimate from a $5\%$ ($5$M) sample} \\
\cmidrule(lr){3-6}
backbone & true mAP@100M & official 5M & naive 5M & count & EVT (ours) \\
 & & (hard-mined) & (random) & & \\
\midrule
DINOv3-L  & 0.266 & 0.311 & 0.374 & 0.317 & \textbf{0.260} \\
PE-core-L & 0.223 & 0.271 & 0.355 & 0.287 & \textbf{0.225} \\
SigLIP2-L & 0.213 & 0.253 & 0.341 & 0.280 & \textbf{0.221} \\
SigLIP-L  & 0.200 & 0.242 & 0.324 & 0.268 & \textbf{0.208} \\
DINOv2-L  & 0.154 & 0.185 & 0.230 & 0.191 & \textbf{0.153} \\
\midrule
mean $|$err$|$ & --- & $+0.041$ & $+0.113$ & $+0.057$ & \textbf{0.005} \\
\bottomrule
\end{tabular}

%% file: sec/discussion.tex
\section{Discussion and Conclusion}
\label{sec:discussion}

The per-query hardness $\{x_q\}$ is itself a reusable artifact: specific to a model and its distractor
source, it lets that method's metric be read off at any gallery size in closed form and makes results
obtained with different distractor budgets commensurable. We release it for our backbones on ILIAS.

\paragraph{Conclusion.}
Web-scale retrieval evaluation is largely predictable: a per-query tail-aware estimator recovers
100M-scale accuracy from a small random distractor sample to \maeMap{} mean error, covering the truth for
every backbone we test on ILIAS, INSTRE, and GLDv2. Validated across five backbones, three benchmarks,
and two distractor pools, it turns web-scale evaluation from an indexing problem into a prediction
problem---for method comparisons whose margins exceed the extrapolation error.

%% file: sec/appendix.tex
\section{Derivation of mAP under distractors}
\label{app:map}

Fix a query $q$ with $n_q$ positives sorted by similarity, $s_{p_1}\ge\dots\ge s_{p_{n_q}}$, so
$p_j$ is the $j$-th best positive. Adding $N$ i.i.d.\ distractors, the number scoring above $p_j$ is
$K_j\sim\mathrm{Binomial}(N,x_{q,j})$ with per-positive hardness $x_{q,j}=1-F_q(s_{p_j})$ (Thm.~1
applied at threshold $s_{p_j}$). The rank of $p_j$ in the full list is the count of items at least
as similar---the $j$ positives $p_1,\dots,p_j$ plus the $K_j$ distractors above it:
\begin{equation}
\mathrm{rank}_j \;=\; j + K_j,\qquad \mathbb{E}[\mathrm{rank}_j]=j+N\,x_{q,j}.
\end{equation}
Average precision truncated at $k$ is, since exactly $j$ positives have rank $\le\mathrm{rank}_j$,
\begin{equation}
\mathrm{AP}_q \;=\; \frac{1}{\min(k,n_q)}\sum_{j:\,\mathrm{rank}_j\le k}\frac{j}{\mathrm{rank}_j}.
\end{equation}
Evaluating the map $R_k$ of Eq.~\eqref{eq:aplaw} at the realized fraction $\bar x_{q,j}=K_j/N$ gives
$\mathrm{rank}_j=j+K_j$ and reproduces this gallery's mAP@$k$ \emph{exactly} (no expectation, no
Taylor remainder; App.~\ref{app:bound}). Feeding instead the population hardness $x_{q,j}$---or its
sample estimate $\hat x_{q,j}$---\emph{predicts} the metric, with the population$\to$realized
concentration error analyzed in App.~\ref{app:bound}. We confirm the identity numerically by
evaluating $\mathrm{rank}_j=j+K_j$ with the observed $K_j$---counting each positive's distractor
exceedances in one streamed pass---and find the analytic form matches the measured mAP
(Tab.~\ref{tab:map}). The protocol substitutes the sampled hardness estimate $\hat x_{q,j}$ (count or
tail-aware) for $x_{q,j}$ and reports the tail-aware estimate with its bracket. Unlike the top-rank survival's
exponential $(1-x_q)^N$, the per-positive precision $j/(j+N\hat x_{q,j})$ is a gently varying
rational response that damps errors in $\hat x_{q,j}$: the tail estimate's over-shoot becomes
negligible, whereas the count estimate stays structurally over-optimistic (it cannot report
hardness below the sample floor $1/S$). Empirically the truth therefore sits at the tail edge of
the interval for mAP@1000---the tail-aware estimate alone is nearly exact, while the count edge
over-predicts (Tab.~\ref{tab:map}).

\section{The tail-aware (EVT) estimator}
\label{app:evt}

For a query whose best positive has similarity $s^\star$, the hardness
$x_q=\Pr[\mathrm{sim}(q,d)>s^\star]$ is the probability that a random distractor outscores the
answer. When $x_q$ is small---the regime that dominates large-$N$ accuracy---no distractor in a
size-$S$ sample reaches $s^\star$, so the empirical count returns $x_q=0$ and we must estimate the
probability of an event that did not occur in the sample.

Extreme-value theory makes this possible without knowing the full score distribution. The
Pickands--Balkema--de Haan theorem states that, for any distribution in an extreme-value domain of
attraction (essentially all continuous distributions), the
\emph{excesses} over a high threshold $u$ (the amounts $\mathrm{sim}-u$ by which scores exceed $u$)
converge to a \emph{generalized-Pareto distribution} (GPD) as $u$ grows. In words: the far tail has
a universal two-parameter shape, independent of the bulk, so we can fit it from the part of the
tail we observe and extend it.

Per query we (1) retain the top $T{=}4000$ distractor similarities---all that bear on the upper tail,
the rest falling far below $s^\star$---and set the threshold $u$ to their $95$th percentile, i.e.\ the
top ${\approx}200$ scores; because $T$ is fixed, this is a \emph{fixed} exceedance count independent of
$S$ (the same ${\approx}200$ at $S{=}50$k or $5$M); (2) fit the GPD shape $\xi$ and scale $\sigma$ to
the excesses $\{\mathrm{sim}_i-u:\mathrm{sim}_i>u\}$ by maximum likelihood; and (3) set the exceedance
rate $\zeta_u=\#\{\mathrm{sim}>u\}/S$---the ${\approx}200$ exceedances normalized by the \emph{full}
sample size $S$, not by $T$, so $\zeta_u$ shrinks as $S$ grows. The hardness estimate is
\begin{equation}
\hat{x}_q \;=\; \zeta_u\,\big(1-\mathrm{GPD}_{\xi,\sigma}(s^\star-u)\big),
\end{equation}
the rate of exceeding $u$ times the fitted-tail probability of reaching all the way to
$s^\star$---read off the curve \emph{beyond} the observed range. We use the empirical count when at least
$c{=}5$ sampled distractors exceed $s^\star$ (Alg.~\ref{alg:protocol}); with fewer, the count is too noisy
(and $c{=}0$ would force $x_q{=}0$), so we read the hardness off the fitted tail instead. This is the
standard peaks-over-threshold construction, used by meta-recognition~\citep{scheirer2011meta,furon2013evt} on
recognition scores; the intuition is estimating a never-observed ``$100$-year flood'' by fitting the
tail of recorded flood heights and extrapolating.

\paragraph{The tail shape must exclude power-law growth.}
The GPD shape $\xi$ sets how heavy the tail is: $\xi>0$ is an infinite, power-law tail, whereas cosine
similarities cannot exceed $1$, so the distractor-score distribution has \emph{bounded} support. By
extreme-value theory a distribution with bounded support lies in the Weibull/Gumbel domain, $\xi\le0$; a
positive $\xi$ would model a physically impossible power-law tail. We therefore cap $\xi\le0$ (refitting
the scale at the boundary when the unconstrained maximum-likelihood fit returns $\xi>0$). This excludes
the over-stated hard mass that otherwise inflates hardness on exactly the near-immune queries that
dominate large-$N$ accuracy. Everything else stays fully adaptive: the threshold $u$, scale $\sigma$, and
exceedance rate $\zeta_u$ are fit per query from its own scores.

\paragraph{Robustness to the threshold.}
The estimator's most consequential knob is the peaks-over-threshold quantile $q$ (we use the $95$th
percentile; $T$ and $\xi\le0$ are fixed once, below). Sweeping it across the five ILIAS backbones, the mAP@100M error is essentially
flat---mean error $0.006$--$0.009$ for every $q\in\{0.90,0.95,0.975,0.99\}$, always below the metric
noise---so the protocol is robust to this choice rather than tuned to it. We fix the full configuration
($T{=}4000$, $q{=}0.95$, $\xi\le0$) once and apply it unchanged to every backbone, dataset, and
pool: it is a single setting, not tuned per benchmark. This estimator configuration is not selected
against any full-scale ($100$M) ground truth; the $5\%$ sample fraction, separately, \emph{is} validated
against it (App.~\ref{app:repeated}). The bounded constraint $\xi\le0$ is dictated by the geometry (cosine similarity is bounded
above by $1$), not fit to the targets; $T$ and $q$ are set on ILIAS and frozen, so INSTRE and GLDv2 are
evaluated with a held-out configuration.
\begin{center}\small
\begin{tabular}{lcccc}
\toprule
threshold $q$ & $0.90$ & $0.95$ & $0.975$ & $0.99$ \\
\midrule
mAP MAE & $0.007$ & $0.006$ & $0.009$ & $0.007$ \\
max error & $0.014$ & $0.010$ & $0.016$ & $0.014$ \\
\bottomrule
\end{tabular}
\end{center}

\paragraph{Small samples, exceedances, and sampling (DINOv3/ILIAS).}
Three further diagnostics. \emph{(i) Sample size.} The estimate is accurate from $S\!\approx\!50$k
distractors (mAP@100M error $0.005$ at $50$k, $\le0.011$ for all $S\ge50$k) and degrades below it
($0.09$ at $10$k, $0.15$ at $5$k)---so below ${\sim}50$k the tail is too sparsely sampled to trust;
${\sim}50$k is the smallest sample at which the estimate is reliable for DINOv3. \emph{(ii) Tail data per query.} The
per-query GPD fit is
well-conditioned: every query supplies ${\approx}200$ exceedances above the threshold (the $95$th
percentile of the retained top-$4000$ tail)---far above the $50$-point minimum, and by construction the
same count at $S\!=\!50$k or $S\!=\!5$M---so per-query EVT is never data-starved.
\emph{(iii) Which distractors.} The prediction is insensitive to \emph{where} the sample is drawn:
across three disjoint $5$M blocks of the $100$M pool it varies by only $0.004$, confirming the pool is
well-mixed and the uniform-sample requirement holds in practice.

\paragraph{GPD goodness of fit (Tab.~\ref{tab:gpd}).}
We audit the per-query fit directly on a $5\%$ sample. The fit never falls back to counting for lack of
tail data (failure rate $\le\gpdFail$). The bounded constraint is not doing all the work: the
unconstrained maximum-likelihood shape is already $\xi\le0$ for most queries, and the cap $\xi\le0$ binds
in only $\gpdCapLo$--$\gpdCapHi$ of them (more often for the heavier-tailed backbones), where it prevents a
physically impossible power-law tail. The median fitted shape is negative and the median fitted upper
endpoint is ${\approx}\gpdEndpoint$---safely below the cosine ceiling $1$, and query-specific rather than
pinned. The one residual concern is the ${\sim}20$--$30\%$ of fits whose endpoint falls below the
positive's own score $s^\star$ (Tab.~\ref{tab:gpd}): there the tail model treats the positive as
effectively immune, which is right for well-separated positives but a coarse approximation for marginal
ones---an error absorbed, like the top-$k$ truncation term, into the end-to-end 100M validation rather
than bounded in closed form.

\begin{table}[t]\centering\small
\input{sec/gpd_table}
\caption{\textbf{Per-query GPD fit diagnostics} on a $5\%$ sample: fraction of queries whose tail is too
small to fit (fail), fraction where the unconstrained MLE wants $\xi{>}0$ so the cap binds, the median
fitted shape $\hat\xi$ and upper endpoint, and the fraction whose endpoint falls below the positive's
score $s^\star$.}
\label{tab:gpd}
\end{table}

\begin{table}[t]\centering\small
\input{sec/estimator_table}
\caption{\textbf{Estimators vs.\ baselines}, MAE of the official GLDv2 mAP@100 at its full $759$k
landmark pool, recovered from a $5\%$ sample and averaged over the \nBackbones{} GLDv2 backbones. The
recsys (K-R) count estimator and curve fits trail tail-aware (EVT) extrapolation.}
\label{tab:baselines}
\end{table}

\begin{table}[t]\centering\small
\input{sec/instre_table}
\caption{\textbf{INSTRE at true 100M, all \nBackbones{} backbones.} True mAP@1000 at the full YFCC-100M
pool and its recovery from a $5\%$ sample. The tail-aware (EVT) estimate is within $0.007$ for every
backbone (mean $|$err$|$ $0.005$) and preserves the full ranking, while the recommender-systems count
edge over-predicts throughout (mean $+0.008$). Same frozen configuration ($T{=}4000$, $q{=}0.95$,
$\xi\le0$) as ILIAS and GLDv2.}
\label{tab:instre}
\end{table}

\begin{figure*}[t]\centering
\includegraphics[width=\linewidth]{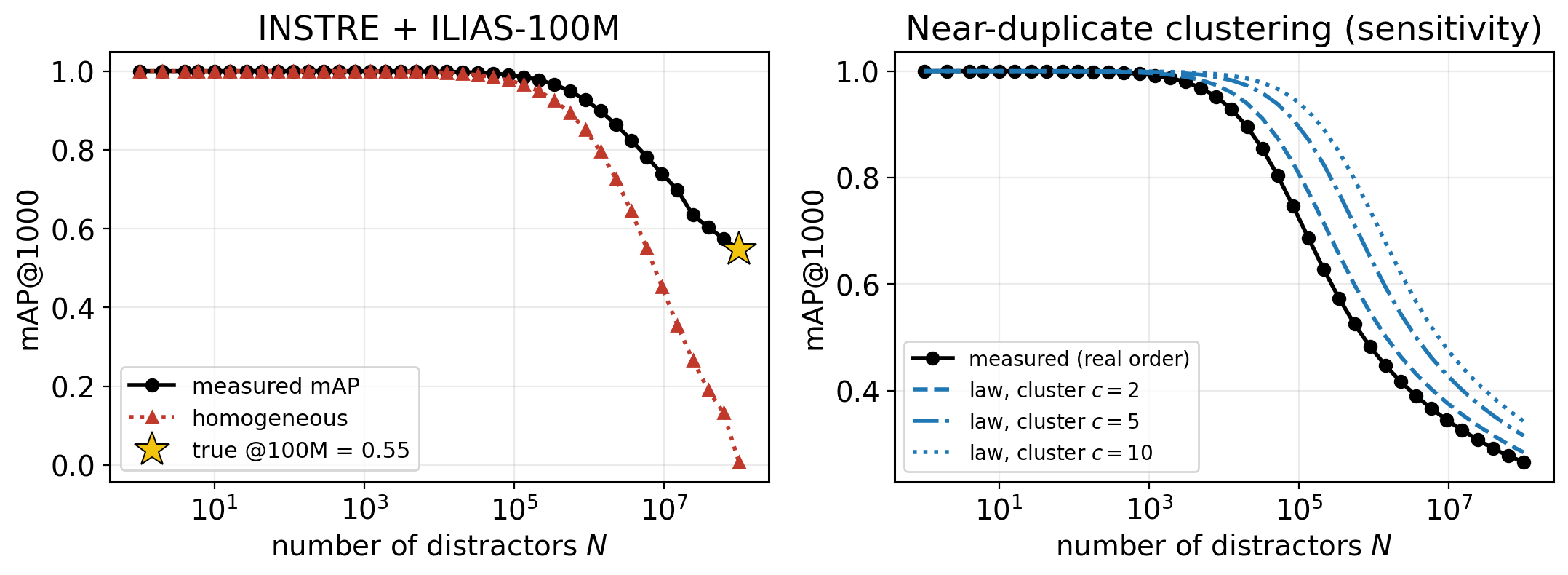}
\caption{\textbf{A second query set and the independence assumption.} \emph{Left:} mAP@1000 (vertical)
vs.\ number of distractors $N$ (horizontal, log) for INSTRE against the ILIAS-100M pool---the law tracks
the truth on a different query set. \emph{Right:} sensitivity to near-duplicate clustering on ILIAS;
treating each cluster of size $c$ as one effective item (top-rank survival $(1-x)^{N/c}$) shifts accuracy
\emph{up}---an illustrative case (a deduplicated gallery), not a general bound on truncated mAP under
retained duplicates. The measured real-order curve sits at $c\!\approx\!1$, so clustering is mild.}
\label{fig:cross}
\end{figure*}

\begin{figure}[t]\centering
\includegraphics[width=0.86\linewidth]{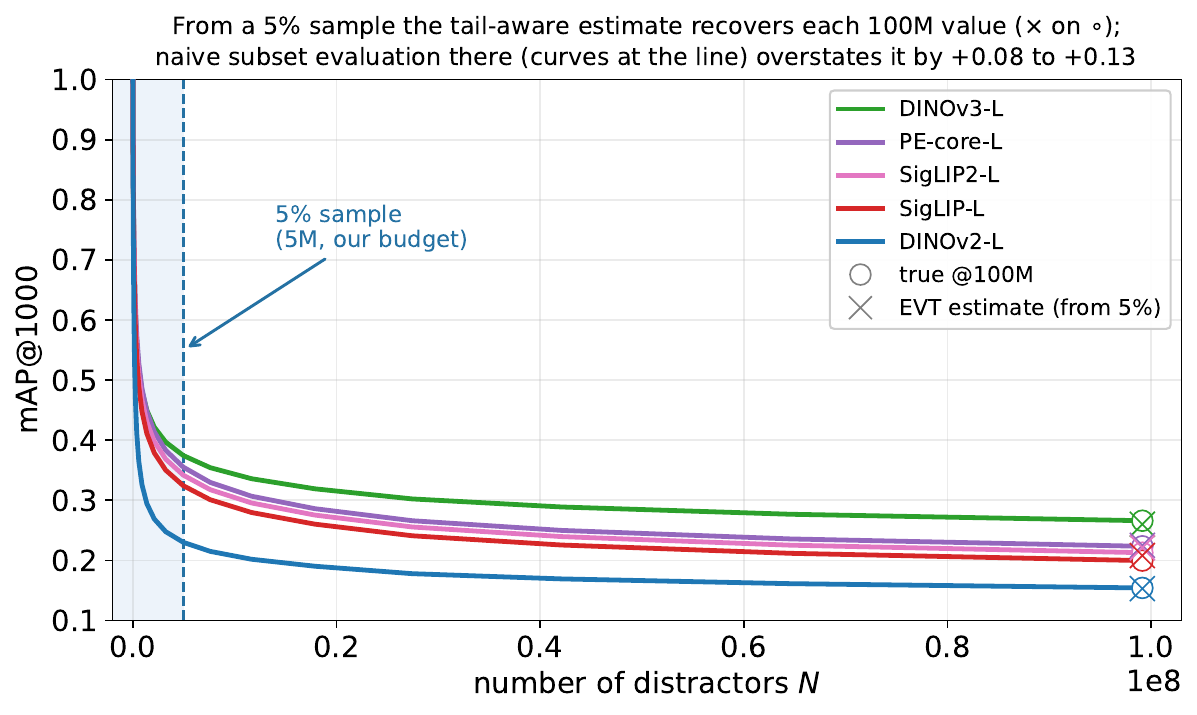}
\caption{\textbf{Recovering mAP@100M from a $5\%$ sample, all five ILIAS backbones.} True mAP@1000 vs.\
number of distractors $N$ (linear axis); each solid curve grows to its true 100M value ($\circ$). Reading
the metric \emph{at} the $5\%$ ($5$M) sample (dashed line; shaded) overstates the 100M value by $+0.08$
to $+0.13$; the tail-aware (EVT) estimate from that same sample ($\times$) recovers each true value to
within ${\sim}0.01$. Estimate markers are the $20$-draw means of Tab.~\ref{tab:repeated}.}
\label{fig:recovery}
\end{figure}

\begin{figure*}[t]\centering
\includegraphics[width=\linewidth]{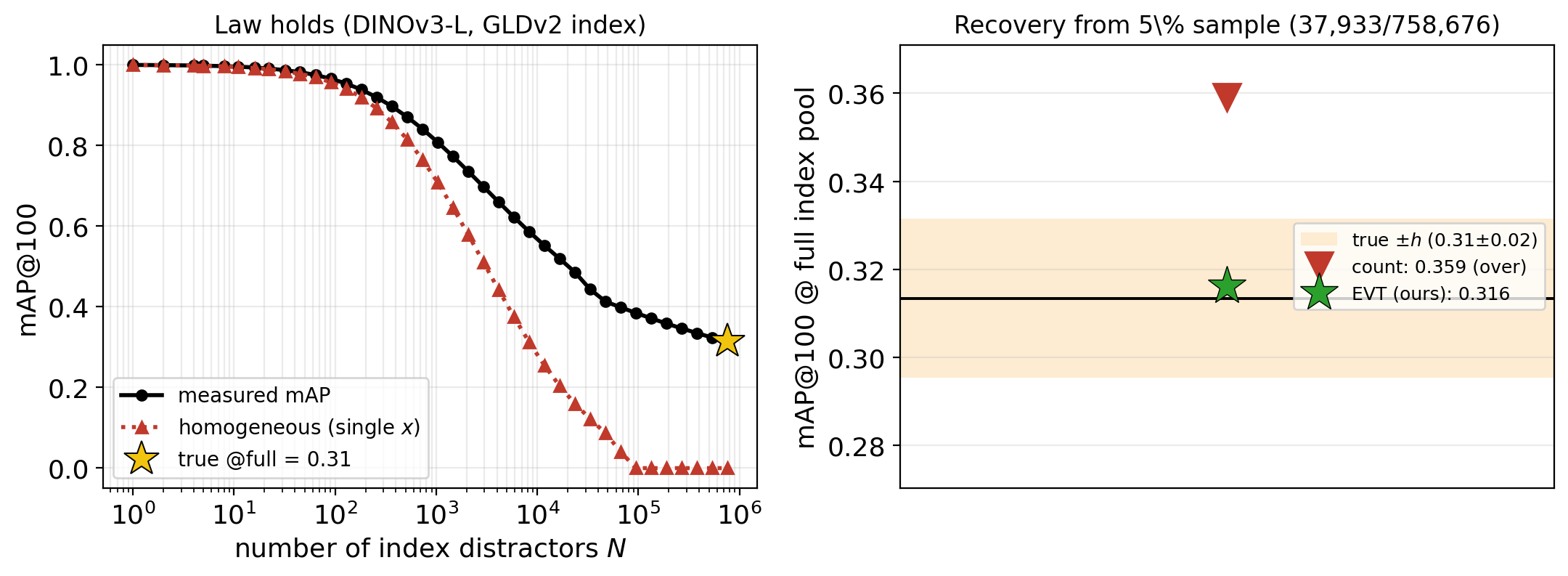}
\caption{\textbf{GLDv2.} \emph{Left:} mAP@100 (vertical) vs.\ number of index distractors $N$
(horizontal, log) for DINOv3-L---the per-query law (dashed) tracks the measured curve (MAE $0.010$).
\emph{Right:} recovery of the official mAP@100 at the full $759$k pool from a $5\%$ sample: the count
edge over-predicts (lands above the band), while the tail-aware (EVT) estimate lands on the truth (black
line), inside the $\pm h$ query-bootstrap noise band (shaded).}
\label{fig:gldv2}
\end{figure*}

\begin{figure}[t]\centering
\includegraphics[width=0.6\linewidth]{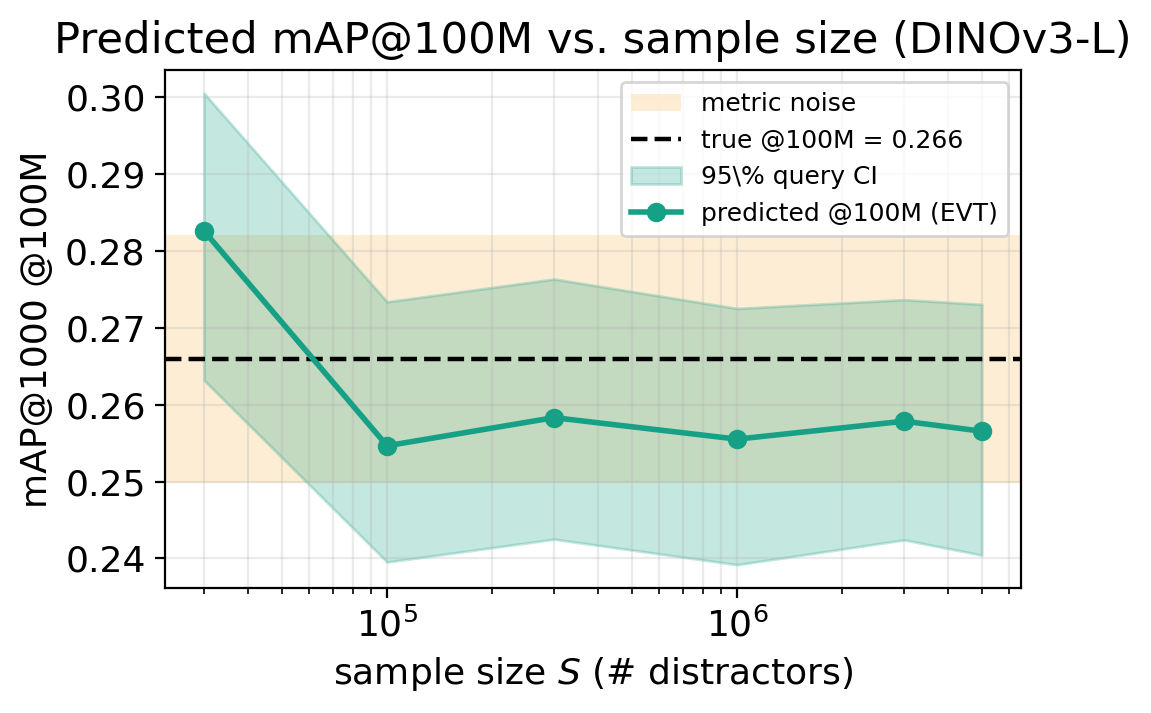}
\caption{Predicted mAP@100M (vertical) vs.\ sample size $S$ (horizontal, log scale) for DINOv3-L,
with a $95\%$ query-bootstrap band; dashed line is the true 100M value.}
\label{fig:sens}
\end{figure}

\paragraph{Why the tail estimate suffices for mAP.}
Empirically the truth sits at the tail edge of the count/tail bracket (Sec.~\ref{sec:protocol}), and
the metric's response to hardness error explains why. The top-rank survival $(1-x_q)^N$ is
exponential and amplifies error in $\hat{x}_q$; mAP's per-positive precision $j/(j+N\hat{x}_{q,j})$ is
a gently varying rational response that damps the tail estimate's over-shoot, so for tail-dominated
reporting scales (mAP@1000 at 100M) the tail estimate alone is nearly exact (Tab.~\ref{tab:map}),
while on shallower pools the count edge over-predicts from above.

\begin{table}[t]\centering\small
\begin{tabular}{llcc}
\toprule
dataset & pool & metric @full pool & EVT err \\
\midrule
ILIAS & YFCC-100M & \trueMap & \maeMap \\
INSTRE & YFCC-100M & \instreHundredMap & $0.005$ \\
GLDv2 & index-759k & $0.31$ & $0.008$ \\
\bottomrule
\end{tabular}
\caption{\textbf{Three benchmarks, two distractor pools.} A small random sample recovers
the large-$N$ mAP within $\sim$0.01 with the tail-aware estimator; ILIAS/INSTRE are mAP@1000, GLDv2
is the official mAP@100. The metric column is DINOv3-L; error is the mean over each pool's \nBackbones{}
backbones (per-backbone detail in Tab.~\ref{tab:eval5m} for ILIAS, Tab.~\ref{tab:instre} for INSTRE,
Tab.~\ref{tab:baselines} for GLDv2).}
\label{tab:cross}
\end{table}

\begin{figure}[t]\centering
\includegraphics[width=0.6\linewidth]{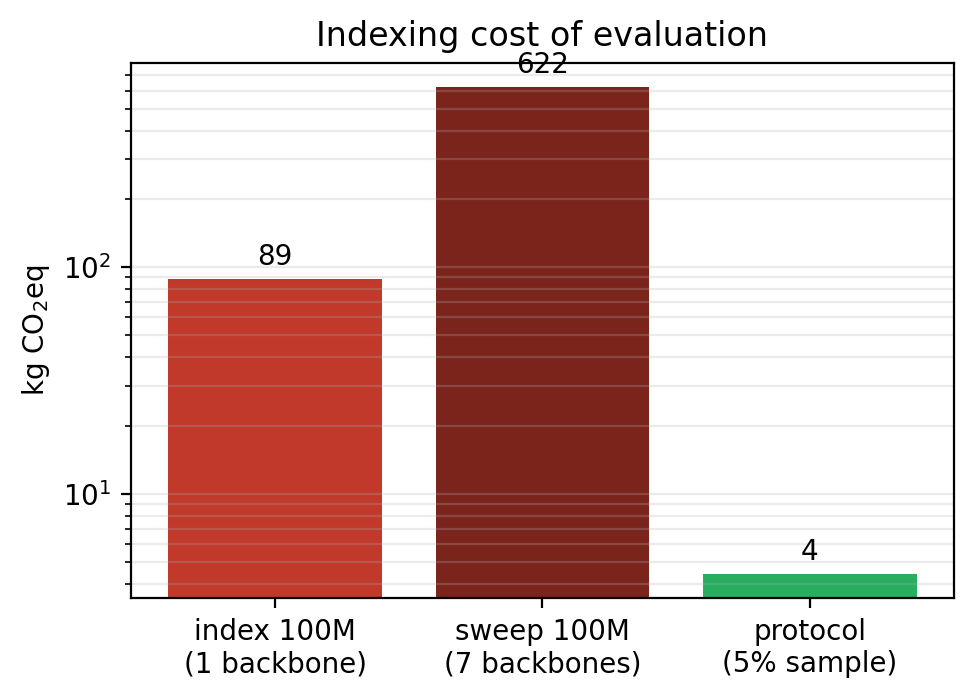}
\caption{Estimated indexing carbon, kg CO$_2$eq (vertical, log scale): indexing 100M for one
backbone ($\approx$\coIndexKg\,kg), a 7-backbone sweep ($\approx$\coSweepKg\,kg), and our $5\%$ protocol
($\approx$\coProtKg\,kg per backbone, $\coReduction\times$ less). Illustrative estimate under the named
hardware assumptions.}
\label{fig:co2}
\end{figure}

\section{Error analysis of the mAP estimator}
\label{app:bound}

\paragraph{The AP law is exact for the realized gallery.}
A benchmark's distractor pool is \emph{one fixed set} of $N$ items, and its mAP is a deterministic
number. The $j$-th positive's rank is exactly $r_j=j+K_j$, where $K_j$ is the actual count of
distractors that outscore it; writing the \emph{realized exceedance fraction} $\bar x_{q,j}:=K_j/N$
(distinct from the population hardness $x_{q,j}$ of Eq.~\eqref{eq:hardness}), we have
$j+N\bar x_{q,j}=j+K_j=r_j$, so
\begin{equation}
\mathrm{mAP}(N)=\frac1{|Q|}\sum_q\frac1{\min(k,n_q)}\sum_{j:\,j+N\bar x_{q,j}\le k}\frac{j}{j+N\bar x_{q,j}}
\end{equation}
is precisely the map $R_k$ of Eq.~\eqref{eq:aplaw} evaluated at the realized fractions---an exact
identity, truncation included, with no linearization term. The $\mathbb{E}[1/r]$ vs.\ $1/\mathbb{E}[r]$
contrast arises only if one treats $K_j$ as a fresh $\mathrm{Poisson}(Nx_{q,j})$ draw and asks for the
mAP \emph{averaged over hypothetical re-samplings of the pool}---a different quantity that no benchmark
reports, and one Jensen's inequality places slightly \emph{above} the realized value. Confirming this across all five backbones: the expected-rank map reproduces the realized
mAP@100M to within $0.001$ (it matches the true value of Tab.~\ref{tab:map} to three decimals), whereas
the exact Binomial $\mathbb{E}[\mathrm{AP}]$---the per-positive sum $\sum_m\tfrac{j}{j+m}\Pr[K_{q,j}{=}m]$,
evaluated in full---sits $+0.004$ to $+0.006$ \emph{above} it. Targeting the realized number with the
expected-rank map, as we do, is therefore the more accurate choice, not an approximation we settle for. The prediction error therefore splits into two terms,
\begin{equation}
\begin{aligned}
\big|\widehat{\mathrm{mAP}}(N)-\mathrm{mAP}(N)\big| \;\le\;
&\underbrace{E_{\mathrm{est}}\big(\{\hat x_{q,j}\}\big)}_{\text{sample}\to\text{pop., at }S}\\
&+\;\underbrace{E_{\mathrm{conc}}}_{\text{pop.}\to\text{realized, at }N},
\end{aligned}
\end{equation}
and the concentration term is negligible: the realized fraction $\bar x_{q,j}=K_j/N$ has standard
deviation $\sqrt{x_{q,j}(1-x_{q,j})/N}$ about $x_{q,j}$---a relative fluctuation $\sqrt{(1-x)/(Nx)}$
that at $N{=}100$M and transition-zone hardness $x\!\sim\!10^{-5}$ is ${\sim}3\%$ per positive and
averages down across positives and queries. The estimation term, at the far smaller sample size $S$,
dominates; we bound it next, separating sample-resolved positives from the unresolved tail.

\paragraph{Estimation term: the visible mass concentrates.}
For a positive whose hardness the sample resolves, the count estimate $\hat x_{q,j}=K/S$ has
$K\sim\mathrm{Binomial}(S,x_{q,j})$, so by Hoeffding's inequality
$|\hat x_{q,j}-x_{q,j}|\le\sqrt{\ln(2/\delta)/(2S)}$ with probability $\ge1-\delta$. This deviation
propagates benignly: the top-rank survival $(1-x)^N$ is $N$-Lipschitz in $x$
($|\tfrac{d}{dx}(1-x)^N|=N(1-x)^{N-1}\le N$), and each AP precision $j/(j+Nx)$ is
$\tfrac1{4x}$-Lipschitz (by AM--GM on $(j+Nx)^2\ge4jNx$), so resolved positives contribute error
controlled by the bound above, and the end-to-end error stays $\le\mapEvtMaxErr{}$ across all five
backbones.

\paragraph{Estimation term: the unresolved tail needs the model.}
The queries that dominate large-$N$ accuracy have $x_q\lesssim1/S$ and are \emph{invisible} in the
sample ($K{=}0$); counting returns $0$, so we estimate their hardness from the fitted upper tail---the
gap posed as open by \citet{madigan2003retrieval}. The bounded GPD is well-suited: by the
Pickands--Balkema--de Haan theorem the threshold-excess distribution approaches a generalized Pareto, and
the cosine ceiling forces $\xi\le0$ (App.~\ref{app:evt}), so the survival decays at least exponentially
and the extrapolated mass $\zeta_u\,(1-\mathrm{GPD}_{\xi,\sigma}(s^\star-u))$ cannot diverge. We validate
the tail estimate empirically---the threshold sweep (App.~\ref{app:evt}), the sample-size and endpoint
diagnostics, and the direct 100M validation. We report the count/tail gap $B$ as an \emph{empirical
diagnostic} of estimator bias, not a guaranteed bound: counting zeroes the unresolved tail and is thus
upward-biased on that mass (a one-sided bias, not a deterministic per-sample bound), while the tail
estimate has no such bias in either direction, so although the truth sits
at the tail edge of $[\hat R_{\mathrm{tail}},\hat R_{\mathrm{count}}]$ in every case we measure, a
misspecified tail could err either way.

\paragraph{Full uncertainty interval: a nested bootstrap that covers $5/5$.}
The interval reported in the main text (Tab.~\ref{tab:map}) resamples the query set only. To fold in
the other two sources---\emph{which} distractors land in the sample, and the per-query EVT
refit on that sample---we run a nested bootstrap on all five ILIAS backbones: draw $K{=}8$ disjoint
$5$M sub-pools of the $100$M stream (disjoint uniform distractor samples), refit the per-query tail
on each, and form the prediction; each bootstrap iteration then picks a random draw (sample $+$ fit
noise) and resamples queries with replacement (query noise). Table~\ref{tab:fullci} reports the
resulting $95\%$ interval against the query-only interval. The added noise is small---the
distractor-sample$+$fit standard deviation is $\sigma_{\mathrm{samp+fit}}\!\approx\!0.003$ (median),
well below the query component ($h/1.96\!\approx\!0.008$)---so the full interval changes little
(endpoints shift more than they widen), and it brackets the held-out true mAP@100M for \emph{all five}
backbones. Query sampling, not distractor sampling or tail-fit, is the dominant uncertainty at the
$5\%$ budget.

\begin{table}[t]\centering\small
\begin{tabular}{lccc}
\toprule
backbone & true mAP@100M & count ($5\%$) & EVT ($5\%$) \\
\midrule
DINOv2-L  & 0.154\,[0.141,0.167] & 0.191\,[0.177,0.205] & 0.153\,[0.140,0.166] \\
SigLIP-L  & 0.200\,[0.185,0.214] & 0.268\,[0.252,0.285] & 0.208\,[0.193,0.222] \\
SigLIP2-L & 0.213\,[0.199,0.228] & 0.280\,[0.263,0.297] & 0.221\,[0.207,0.236] \\
PE-core-L & 0.223\,[0.209,0.238] & 0.287\,[0.271,0.304] & 0.225\,[0.211,0.240] \\
DINOv3-L  & 0.266\,[0.250,0.283] & 0.317\,[0.300,0.335] & 0.260\,[0.244,0.277] \\
\bottomrule
\end{tabular}
\caption{\textbf{Detailed $95\%$ query-bootstrap intervals for Tab.~\ref{tab:map}.} The true value, the
tail-blind count estimate, and our tail-aware EVT estimate, each with a $95\%$ query-bootstrap interval
(resampling the $1{,}232$ queries). The EVT interval overlaps the true interval for every backbone ($5/5$
coverage); the count interval lies entirely \emph{above} the truth for every backbone ($0/5$)---its error
is a systematic bias, not query-sampling noise.}
\label{tab:mapci}
\end{table}

\begin{table*}[t]\centering\small
\input{sec/fullci_table}
\caption{\textbf{Full predictive interval (nested bootstrap).} Folding distractor-sample and EVT-fit
noise into the query-bootstrap interval (via $K{=}8$ disjoint $5$M redraws with per-draw tail refits)
adds only $\sigma_{\mathrm{samp+fit}}\!\approx\!0.003$ std---the endpoints shift slightly rather
than always widen---and still covers the true 100M mAP for every backbone ($5/5$).}
\label{tab:fullci}
\end{table*}

\paragraph{Ranking preservation and paired error.}
Interval overlap (Tab.~\ref{tab:map}) is a weak coverage test, since prediction and truth share the
same queries. Two stronger statistics confirm the estimate. \emph{(i) Ranking.} Across the five
backbones the predicted and true mAP@1000 induce the \emph{same} ordering for the reported estimate:
Kendall $\tau=1.0$, Spearman $\rho=1.0$, no sign flip among the ten pairwise margins. Across $\repK$
\emph{disjoint} distractor samples this holds for the nine pairs separated by more than the
prediction error (preserved in every draw); only the tightest pair (PE-core vs.\ SigLIP2, true gap
$\repContestGap$, within the error) is not reliably resolved, being statistically tied at 100M
(App.~\ref{app:repeated}). Ranking is recovered down to a margin of ${\approx}\repResolveGap$ mAP. \emph{(ii) Paired error.} The
per-query difference $\mathrm{AP}^{\mathrm{EVT}}_q-\mathrm{AP}^{\mathrm{true}}_q$ has pooled mean
$+0.002$ with a $95\%$ bootstrap CI of $[-0.001,+0.004]$ over the $6160$ query-backbone cells; per
backbone the bias stays within $\pm0.008$ (largest SigLIP/SigLIP2 $+0.008$, DINOv3 $-0.006$).

\paragraph{Equal-cost comparison: evaluate at 5M vs.\ predict from 5M.}
At the same $5$M indexing cost ($5\%$ of the pool), a practitioner can either \emph{evaluate} on the 5M
subset (the incumbent cheap practice, e.g.\ the hard-mined mini-ILIAS~\citep{ilias2025}) or
\emph{predict} the 100M metric from it (ours). Naive evaluation is optimistic by $+0.08$ to $+0.13$
(mean $|$err$|$ $0.113$), because 5M distractors under-rank the positives. The tail-aware prediction is
within $0.005$ mean (Tab.~\ref{tab:eval5m}). At identical cost, predicting is roughly $22\times$ more
accurate than evaluating. Evaluate-at-5M is the expected value of a uniform 5M subset (the AP law at
$N{=}5$M, which reproduces the measured DINOv2 curve to $0.006$); a hard-mined 5M set such as
mini-ILIAS is non-uniform and reports a different biased number still short of the 100M value.

\begin{table}[t]\centering\small
\setlength{\tabcolsep}{4pt}
\input{sec/eval5m_table}
\caption{\textbf{At equal 5M cost, predict beats evaluate.} Naive mAP@1000 measured on a uniform 5M
subset (evaluate-at-5M) is optimistic by $+0.08$ to $+0.13$; the tail-aware prediction from the same 5M
sample (EVT-at-5M) is within $0.005$ of the true 100M value. Parentheses are signed error vs.\ true@100M.}
\label{tab:eval5m}
\end{table}

\section{Repeated-sample validation: sampling distribution, coverage, and calibration}
\label{app:repeated}

The interval in the main text is a query bootstrap from a \emph{single} distractor sample; the stronger
test resamples the \emph{distractors}. We partition each backbone's $99$M pool into $K{=}\repK$ disjoint,
pool-spread blocks at a given sample fraction and predict mAP@100M on each---$\repK$ genuinely
disjoint draws (not a single-sample bootstrap; disjoint blocks of a finite pool are negatively dependent under the finite-population constraint), using the paper's estimator ($T{=}4000$, $q{=}0.95$,
$\xi\le0$).

\paragraph{Coverage at the $5\%$ operating point (Tab.~\ref{tab:repeated}).}
Over the $\repK$ disjoint $5\%$ draws per backbone the mean MAE is $\repMae$ ($\repMaeLo$--$\repMaeHi$
across backbones), the draw-to-draw standard deviation is $\repStd$---matching the nested-bootstrap
$\sigma_{\mathrm{samp+fit}}$ of App.~\ref{app:bound}---and the worst single error is $\repWorst$.
\emph{Every} draw lands inside the reported query-bootstrap interval: coverage is $\repCover$ (all $100$
backbone-draws). The interval is, if anything, conservative---its half-width is set by the larger
query-sampling uncertainty, whereas the distractor-sample spread measured directly here is only
${\approx}\repStd$.

\begin{table}[t]\centering\small
\resizebox{\columnwidth}{!}{%
\input{sec/repeated_table}}
\caption{\textbf{Interval coverage over $\repK$ disjoint $5\%$ samples.} Each backbone's $99$M pool is
split into $\repK$ disjoint $5\%$ blocks; mean prediction, MAE, draw-to-draw std, worst error, and the
number of draws inside the reported interval ($|$err$|{\le}h$). Coverage is $\repCover$ overall.}
\label{tab:repeated}
\end{table}

\paragraph{The evaluation converges within a $5\%$ subset (Tab.~\ref{tab:sweep}).}
$5\%$ is a result, not a tuned operating point: as the sample grows, every backbone's estimate converges
onto its $100$M value. Mean MAE falls from $0.020$ at $0.5\%$ to $\repMae$ at $5\%$. The well-separated
backbones (DINOv2, DINOv3) converge by $0.5\%$; the heavy-tailed backbones (SigLIP, SigLIP2), whose deeper
tail a smaller sample cannot yet resolve, converge later---their interval reaches the true value only by
$5\%$ ($0/\repK$ covered at $0.5$--$1\%$). Every tested model converges within a mere $5\%$ of the pool.

\begin{table}[t]\centering\small
\resizebox{\columnwidth}{!}{%
\input{sec/sweep_table}}
\caption{\textbf{Accuracy vs.\ sample fraction} (mean over the five backbones, disjoint draws per row, at
the paper's $T{=}4000$). Error and worst case shrink monotonically with the sample; $5\%$ is the point at
which every backbone is covered.}
\label{tab:sweep}
\end{table}

\paragraph{The sufficient budget.}
Choose a resolution $\tau$. Growing the distractor sample in geometric steps, we predict mAP@100M with the
estimator at each size and stop once doubling from $n/2$ to $n$ moves the prediction by no more than
$\tau$; that sample size is the sufficient budget. Taking $\tau$ at the metric's own query-sampling noise
$h{=}1.96\,\mathrm{std}_q(\mathrm{AP}_q)/\sqrt{n_q}$ ($0.013$--$0.019$) and running against the $100$M
ground truth across the ten web-scale cells (ILIAS and INSTRE, five backbones each, $10^8$ distractors),
the estimate reaches its sufficient budget within $1\%$ of the pool, and at $5\%$ it matches the true
$100$M mAP to a mean absolute error of $0.005$. The $5\%$ budget is more than sufficient to predict the
$100$M metric, and the information that determines it is present in a much smaller sample.

\paragraph{Per-positive calibration (Fig.~\ref{fig:calib}).}
Fig.~\ref{fig:calib} plots the estimated per-positive hardness $\hat x_{q,j}$ (from a $5\%$ sample)
against the true $x_{q,j}$ (100M). Where the sample resolves the positive (count ${\ge}5$) the estimate
sits on the diagonal; in the deep tail ($x<1/S$) the EVT extrapolation is noisier per positive
(log-correlation ${\approx}0.4$) but \emph{unbiased in the median} (median $\hat x/x$ between $0.9$ and
$1.3$ across backbones). Because mAP averages over ${\sim}6000$ positives, this per-positive tail noise
cancels and the metric is recovered to $\repMae$.

\begin{figure*}[t]\centering
\includegraphics[width=\textwidth]{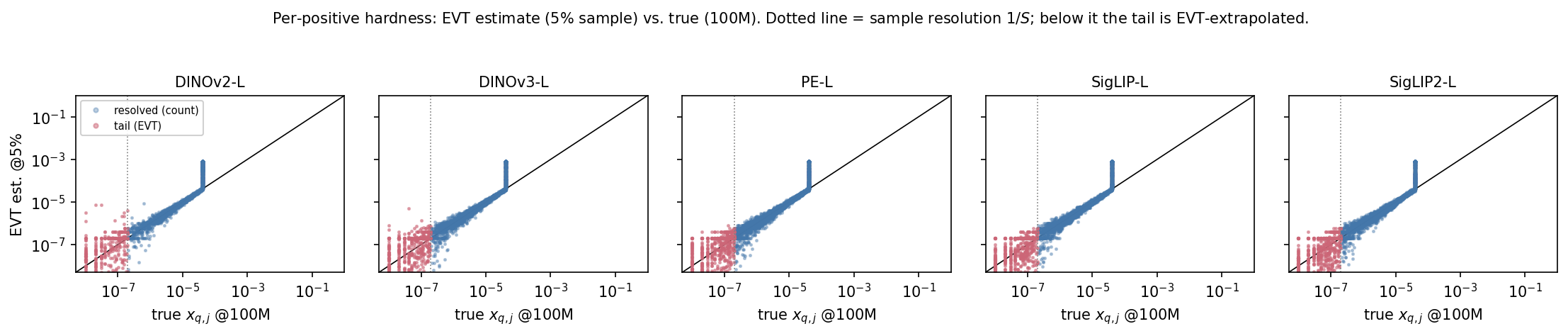}
\caption{\textbf{Per-positive hardness: estimate vs.\ truth.} EVT estimate from a $5\%$ sample (vertical)
vs.\ the true value at 100M (horizontal), per backbone, log--log. Blue: positives the sample resolves by
counting; red: the deep tail ($x<1/S$, dotted line) reached only by EVT extrapolation. The tail scatters
about the diagonal---noisy per positive, unbiased in the median---so the aggregated mAP is accurate.}
\label{fig:calib}
\end{figure*}

\paragraph{Ranking resolution.}
Across the $\repK$ disjoint draws, $\repPairsPerfect$ of the ten backbone pairs are ordered correctly
in \emph{every} draw; the smallest true gap resolved in all draws is $\repResolveGap$. The only pair that
ever swaps is PE-core vs.\ SigLIP2 (true gap $\repContestGap$, below the prediction error), which are
statistically tied at 100M. The method thus resolves rankings down to a margin of ${\approx}\repResolveGap$
mAP: pairs wider than that are preserved across disjoint samples; pairs within the error are not, and
should be reported as ties.

%% file: sec/gpd_table.tex
\begin{tabular}{lccccc}
\toprule
backbone & fit & $\xi{>}0$ & med & med & endpt \\
 & fail & (cap) & $\hat\xi$ & endpt & ${<}s^\star$ \\
\midrule
DINOv2-L & 0.0\% & 10\% & -0.14 & 0.77 & 22\% \\
SigLIP-L & 0.0\% & 29\% & -0.06 & 0.86 & 22\% \\
SigLIP2-L & 0.0\% & 30\% & -0.06 & 0.89 & 20\% \\
PE-core-L & 0.0\% & 21\% & -0.09 & 0.84 & 25\% \\
DINOv3-L & 0.0\% & 14\% & -0.11 & 0.75 & 31\% \\
\bottomrule
\end{tabular}

%% file: sec/estimator_table.tex
\begin{tabular}{lc}
\toprule
estimator & mAP@100 MAE \\
\midrule
raw / evaluate-at-$S$ & 0.038 \\
count (K-R) & 0.038 \\
log-linear & 0.111 \\
log-quadratic & 0.095 \\
power-law & 0.117 \\
log-uniform fit & 0.075 \\
spline survival (PCHIP) & 0.043 \\
Beta-Binomial (emp.\ Bayes) & 0.164 \\
Weibull tail, endpoint$=1$ & 0.012 \\
\textbf{EVT (ours)} $\star$ & \textbf{0.008} \\
\bottomrule
\end{tabular}

%% file: sec/instre_table.tex
\begin{tabular}{lccc}
\toprule
backbone & true@100M & EVT-at-$5\%$ & count-at-$5\%$ \\
\midrule
PE-core-L & 0.613 & 0.607\,($-0.006$) & 0.621\,($+0.008$) \\
DINOv3-L  & 0.547 & 0.540\,($-0.007$) & 0.555\,($+0.008$) \\
SigLIP-L  & 0.532 & 0.528\,($-0.004$) & 0.541\,($+0.009$) \\
SigLIP2-L & 0.460 & 0.454\,($-0.006$) & 0.469\,($+0.009$) \\
DINOv2-L  & 0.355 & 0.352\,($-0.003$) & 0.363\,($+0.008$) \\
\midrule
mean $|$err$|$ & --- & $0.005$ & $0.008$ \\
\bottomrule
\end{tabular}

%% file: sec/fullci_table.tex
\begin{tabular}{lcccc}
\toprule
backbone & true mAP@100M & query-only 95\% & full 95\% (q+samp+fit) & $\sigma_{\mathrm{samp+fit}}$ \\
\midrule
DINOv3-L  & 0.266 & [0.240,\,0.273] & [0.243,\,0.277] & 0.002 \\
DINOv2-L  & 0.154 & [0.140,\,0.166] & [0.139,\,0.167] & 0.003 \\
SigLIP-L  & 0.200 & [0.193,\,0.221] & [0.192,\,0.223] & 0.003 \\
SigLIP2-L & 0.213 & [0.207,\,0.238] & [0.203,\,0.236] & 0.004 \\
PE-core-L & 0.223 & [0.209,\,0.238] & [0.209,\,0.240] & 0.003 \\
\bottomrule
\end{tabular}

%% file: sec/eval5m_table.tex
\begin{tabular}{lccc}
\toprule
backbone & eval-at-5M & EVT-at-5M & true@100M \\
\midrule
DINOv2-L  & 0.230\,($+0.076$) & 0.153\,($-0.001$) & 0.154 \\
SigLIP-L  & 0.324\,($+0.124$) & 0.208\,($+0.008$) & 0.200 \\
SigLIP2-L & 0.341\,($+0.128$) & 0.221\,($+0.008$) & 0.213 \\
PE-core-L & 0.355\,($+0.131$) & 0.225\,($+0.002$) & 0.223 \\
DINOv3-L  & 0.374\,($+0.108$) & 0.260\,($-0.006$) & 0.266 \\
\midrule
mean $|$err$|$ & $0.113$ & $0.005$ & --- \\
\bottomrule
\end{tabular}

%% file: sec/repeated_table.tex
\begin{tabular}{lcccccc}
\toprule
backbone & true & mean pred & MAE & std & worst & cover \\
 & mAP@100M & (5\%) & & & $|$err$|$ & $|$err$|{\le}h$ \\
\midrule
DINOv2-L & 0.154 & 0.153 & 0.002 & 0.002 & 0.007 & 20/20 \\
SigLIP-L & 0.200 & 0.208 & 0.008 & 0.003 & 0.013 & 20/20 \\
SigLIP2-L & 0.213 & 0.221 & 0.008 & 0.003 & 0.012 & 20/20 \\
PE-core-L & 0.223 & 0.225 & 0.002 & 0.003 & 0.008 & 20/20 \\
DINOv3-L & 0.266 & 0.260 & 0.006 & 0.002 & 0.010 & 20/20 \\
\bottomrule
\end{tabular}

%% file: sec/sweep_table.tex
\begin{tabular}{lccccc}
\toprule
fraction & $n$ & draws & mean MAE & worst & mean std \\
 & & (disjoint) & (5 bb) & $|$err$|$ & \\
\midrule
0.5\% & 0.50M & 12 & 0.020 & 0.039 & 0.003 \\
1.0\% & 0.99M & 12 & 0.015 & 0.031 & 0.003 \\
2.0\% & 1.98M & 12 & 0.010 & 0.023 & 0.003 \\
5.0\% & 4.96M & 20 & 0.005 & 0.013 & 0.003 \\
\bottomrule
\end{tabular}

%% file: main.bbl
\begin{thebibliography}{29}
\providecommand{\natexlab}[1]{#1}
\providecommand{\url}[1]{\texttt{#1}}
\expandafter\ifx\csname urlstyle\endcsname\relax
  \providecommand{\doi}[1]{doi: #1}\else
  \providecommand{\doi}{doi: \begingroup \urlstyle{rm}\Url}\fi

\bibitem[Bahri et~al.(2023)Bahri, Zheng, Tay, Metzler, and
  Tomkins]{bahri2020surprise}
Dara Bahri, Che Zheng, Yi~Tay, Donald Metzler, and Andrew Tomkins.
\newblock Surprise: Result list truncation via extreme value theory.
\newblock In \emph{SIGIR}, pp.\  2404--2408, 2023.

\bibitem[Buckley \& Voorhees(2004)Buckley and Voorhees]{buckley2004incomplete}
Chris Buckley and Ellen~M Voorhees.
\newblock Retrieval evaluation with incomplete information.
\newblock In \emph{SIGIR}, 2004.

\bibitem[Carterette et~al.(2006)Carterette, Allan, and
  Sitaraman]{carterette2006minimal}
Ben Carterette, James Allan, and Ramesh Sitaraman.
\newblock Minimal test collections for retrieval evaluation.
\newblock In \emph{SIGIR}, 2006.

\bibitem[Fang et~al.(2024)Fang, Zhan, Ai, Mao, Su, Chen, and
  Liu]{fang2024scaling}
Yan Fang, Jingtao Zhan, Qingyao Ai, Jiaxin Mao, Weihang Su, Jia Chen, and Yiqun
  Liu.
\newblock Scaling laws for dense retrieval.
\newblock In \emph{SIGIR}, 2024.

\bibitem[Furon \& J\'egou(2013)Furon and J\'egou]{furon2013evt}
Teddy Furon and Herv\'e J\'egou.
\newblock Using extreme value theory for image detection.
\newblock Technical Report RR-8244, INRIA, 2013.

\bibitem[Grother \& Phillips(2004)Grother and Phillips]{grother2004models}
Patrick Grother and P.~Jonathon Phillips.
\newblock Models of large population recognition performance.
\newblock In \emph{CVPR}, 2004.

\bibitem[Hawking \& Robertson(2003)Hawking and
  Robertson]{hawking2003collection}
David Hawking and Stephen Robertson.
\newblock On collection size and retrieval effectiveness.
\newblock \emph{Information Retrieval}, 6\penalty0 (1):\penalty0 99--150, 2003.

\bibitem[Hestness et~al.(2017)Hestness, Narang, Ardalani, Diamos, Jun,
  Kianinejad, Patwary, Yang, and Zhou]{hestness2017scaling}
Joel Hestness, Sharan Narang, Newsha Ardalani, Gregory Diamos, Heewoo Jun,
  Hassan Kianinejad, Md. Mostofa~Ali Patwary, Yang Yang, and Yanqi Zhou.
\newblock Deep learning scaling is predictable, empirically.
\newblock \emph{arXiv:1712.00409}, 2017.

\bibitem[Jin et~al.(2021)Jin, Li, Mudrak, Gao, and Liu]{jin2021estimating}
Ruoming Jin, Dong Li, Benjamin Mudrak, Jing Gao, and Zhi Liu.
\newblock On estimating recommendation evaluation metrics under sampling.
\newblock In \emph{AAAI}, 2021.

\bibitem[Kaplan et~al.(2020)Kaplan, McCandlish, Henighan, Brown, Chess, Child,
  Gray, Radford, Wu, and Amodei]{kaplan2020scaling}
Jared Kaplan, Sam McCandlish, Tom Henighan, Tom~B. Brown, Benjamin Chess, Rewon
  Child, Scott Gray, Alec Radford, Jeffrey Wu, and Dario Amodei.
\newblock Scaling laws for neural language models.
\newblock \emph{arXiv:2001.08361}, 2020.

\bibitem[Kemelmacher-Shlizerman et~al.(2016)Kemelmacher-Shlizerman, Seitz,
  Miller, and Brossard]{megaface2016}
Ira Kemelmacher-Shlizerman, Steven~M. Seitz, Daniel Miller, and Evan Brossard.
\newblock The {MegaFace} benchmark: 1 million faces for recognition at scale.
\newblock In \emph{CVPR}, 2016.

\bibitem[Kordopatis-Zilos et~al.(2025)Kordopatis-Zilos, Stojni\'c, Manko,
  \v{S}uma, Ypsilantis, Efthymiadis, Laskar, Matas, Chum, and
  Tolias]{ilias2025}
Giorgos Kordopatis-Zilos, Vladan Stojni\'c, Anna Manko, Pavel \v{S}uma,
  Nikolaos-Antonios Ypsilantis, Nikos Efthymiadis, Zakaria Laskar, Ji\v{r}\'i
  Matas, Ond\v{r}ej Chum, and Giorgos Tolias.
\newblock {ILIAS}: Instance-level image retrieval at scale.
\newblock In \emph{CVPR}, 2025.

\bibitem[Krichene \& Rendle(2020)Krichene and Rendle]{krichene2020sampled}
Walid Krichene and Steffen Rendle.
\newblock On sampled metrics for item recommendation.
\newblock In \emph{KDD}, 2020.

\bibitem[Lacoste et~al.(2019)Lacoste, Luccioni, Schmidt, and
  Dandres]{lacoste2019co2}
Alexandre Lacoste, Alexandra Luccioni, Victor Schmidt, and Thomas Dandres.
\newblock Quantifying the carbon emissions of machine learning.
\newblock \emph{arXiv:1910.09700}, 2019.

\bibitem[Li et~al.(2023)Li, Jin, Liu, Ren, Gao, and Liu]{li2023reliable}
Dong Li, Ruoming Jin, Zhenming Liu, Bin Ren, Jing Gao, and Zhi Liu.
\newblock Towards reliable item sampling for recommendation evaluation.
\newblock In \emph{AAAI}, 2023.

\bibitem[Madigan et~al.(2003)Madigan, Vardi, and
  Weissman]{madigan2003retrieval}
David Madigan, Yehuda Vardi, and Ishay Weissman.
\newblock On retrieval properties of samples of large collections.
\newblock In \emph{AISTATS}, 2003.

\bibitem[Madigan et~al.(2006)Madigan, Vardi, and Weissman]{madigan2006evt}
David Madigan, Yehuda Vardi, and Ishay Weissman.
\newblock Extreme value theory applied to document retrieval from large
  collections.
\newblock \emph{Information Retrieval}, 9:\penalty0 273--294, 2006.

\bibitem[{NIST}(2022)]{nistfrvt}
{NIST}.
\newblock Face recognition vendor test ({FRVT}): Relating 1:1 and 1:n false
  positive rates.
\newblock Technical report, National Institute of Standards and Technology,
  2022.
\newblock FRVT report card.
  \url{https://pages.nist.gov/frvt/reports/demographics/implications_for_1N.pdf}.

\bibitem[Polo et~al.(2024)Polo, Weber, Choshen, Sun, Xu, and
  Yurochkin]{polo2024tiny}
Felipe~Maia Polo, Lucas Weber, Leshem Choshen, Yuekai Sun, Gongjun Xu, and
  Mikhail Yurochkin.
\newblock {tinyBenchmarks}: Evaluating {LLMs} with fewer examples.
\newblock In \emph{ICML}, 2024.

\bibitem[Radenovi\'c et~al.(2018)Radenovi\'c, Iscen, Tolias, Avrithis, and
  Chum]{radenovic2018revisiting}
Filip Radenovi\'c, Ahmet Iscen, Giorgos Tolias, Yannis Avrithis, and Ond\v{r}ej
  Chum.
\newblock Revisiting {Oxford} and {Paris}: Large-scale image retrieval
  benchmarking.
\newblock In \emph{CVPR}, 2018.

\bibitem[Rocher et~al.(2025)Rocher, Hendrickx, and
  de~Montjoye]{rocher2025scaling}
Luc Rocher, Julien~M Hendrickx, and Yves-Alexandre de~Montjoye.
\newblock A scaling law to model the effectiveness of identification
  techniques.
\newblock \emph{Nature Communications}, 2025.

\bibitem[Scheirer et~al.(2011)Scheirer, Rocha, Micheals, and
  Boult]{scheirer2011meta}
Walter~J Scheirer, Anderson Rocha, Ross~J Micheals, and Terrance~E Boult.
\newblock Meta-recognition: The theory and practice of recognition score
  analysis.
\newblock \emph{IEEE TPAMI}, 33\penalty0 (8), 2011.

\bibitem[Slavutsky \& Benjamini(2021)Slavutsky and
  Benjamini]{slavutsky2021predicting}
Yuli Slavutsky and Yuval Benjamini.
\newblock Predicting classification accuracy when adding new unobserved
  classes.
\newblock In \emph{ICLR}, 2021.

\bibitem[Vivek et~al.(2024)Vivek, Ethayarajh, Yang, and Kiela]{vivek2024anchor}
Rajan Vivek, Kawin Ethayarajh, Diyi Yang, and Douwe Kiela.
\newblock Anchor points: Benchmarking models with much fewer examples.
\newblock In \emph{EACL}, 2024.

\bibitem[Wang \& Jiang(2015)Wang and Jiang]{instre2015}
Shuang Wang and Shuqiang Jiang.
\newblock {INSTRE}: A new benchmark for instance-level object retrieval and
  recognition.
\newblock \emph{ACM Transactions on Multimedia Computing, Communications, and
  Applications}, 11\penalty0 (3), 2015.

\bibitem[Weyand et~al.(2020)Weyand, Araujo, Cao, and Sim]{weyand2020gldv2}
Tobias Weyand, Andre Araujo, Bingyi Cao, and Jack Sim.
\newblock {Google Landmarks Dataset v2}: A large-scale benchmark for
  instance-level recognition and retrieval.
\newblock In \emph{CVPR}, 2020.

\bibitem[Yilmaz \& Aslam(2006)Yilmaz and Aslam]{yilmaz2006infap}
Emine Yilmaz and Javed~A Aslam.
\newblock Estimating average precision with incomplete and imperfect judgments.
\newblock In \emph{CIKM}, 2006.

\bibitem[Yilmaz et~al.(2008)Yilmaz, Kanoulas, and Aslam]{yilmaz2008statap}
Emine Yilmaz, Evangelos Kanoulas, and Javed~A Aslam.
\newblock A simple and efficient sampling method for estimating {AP} and
  {NDCG}.
\newblock In \emph{SIGIR}, 2008.

\bibitem[Zheng et~al.(2018)Zheng, Achanta, and
  Benjamini]{zheng2018extrapolating}
Charles Zheng, Rakesh Achanta, and Yuval Benjamini.
\newblock Extrapolating expected accuracies for large multi-class problems.
\newblock \emph{Journal of Machine Learning Research}, 19\penalty0
  (65):\penalty0 1--30, 2018.

\end{thebibliography}
